\makeatletter
\providecommand\ClearHookNext[1]{}
\makeatother
\documentclass[review,onefignum,onetabnum]{siamart251216}
\makeatletter
\providecommand\ClearHookNext[1]{}
\makeatother

\usepackage{iftex}
\ifPDFTeX
  \usepackage[utf8]{inputenc}
  \usepackage[T1]{fontenc}
  \usepackage{lmodern}
\fi

\usepackage{setspace}
\usepackage{microtype}

\usepackage{amsmath,amssymb,amsfonts,amsthm,mathtools}
\usepackage{bm}
\usepackage{bbm}

\usepackage{booktabs}
\usepackage{array}
\usepackage{enumitem}
\usepackage{algorithm}
\usepackage{algorithmic}
\usepackage{float}

\usepackage{graphicx}
\usepackage{xcolor}
\usepackage{tikz}
\usetikzlibrary{arrows.meta}

\usepackage[numbers,sort&compress]{natbib}
\usepackage[colorlinks=true,linkcolor=blue,citecolor=red,urlcolor=blue]{hyperref}
\usepackage[nameinlink,capitalize,noabbrev]{cleveref}
\numberwithin{equation}{section}
\newtheorem{definition}{Definition}[section]
\newtheorem{assumption}{Assumption}[section]
\newtheorem{proposition}{Proposition}[section]
\newtheorem{lemma}{Lemma}[section]
\newtheorem{theorem}{Theorem}[section]

\newcommand{\R}{\mathbb{R}}

\newcommand{\jsr}{\rho}

\newcommand{\set}[1]{\left\{ #1 \right\}}
\newcommand{\co}{\operatorname{co}}
\newcommand{\diag}{\operatorname{diag}}

\title{Target Updates May Stabilize Linear Q-Learning:\\
Periodic and Soft Dynamics}

\author{Donghwan Lee\thanks{School of Electrical Engineering, KAIST, Daejeon, Republic of Korea (\email{donghwan@kaist.ac.kr}).}}

\headers{Target Updates and Linear Q-Learning}{D. Lee}

\date{}

\begin{document}

\maketitle

\begin{abstract}
Periodic target updates in Q-learning and soft target updates in actor-critic
methods are empirically well established stabilization mechanisms, but their
precise theoretical explanation is still incomplete. This paper gives a rigorous
and exact analysis of these mechanisms for Q-learning with linear function
approximation (linear Q-learning) using the exact switched linear system (SLS) dynamics
induced by the Bellman maximum and the joint spectral radius (JSR) of the
resulting switching matrix families. Although linear Q-learning can fail to converge in
general, we prove that, under explicit spectral and step-size conditions,
periodic hard target updates and soft target updates can guarantee convergence
to the exact projected Q-Bellman solution. The main analysis is carried out for
deterministic linear Q-learning, where the target-update mechanism is most
transparent. Once the corresponding JSR certificate is established for the mean recursion,
the stochastic reinforcement-learning setting can be treated by replacing
deterministic modes with sampled stochastic modes and adding the corresponding
stochastic-noise analysis.
\end{abstract}



\section{Introduction}
\label{sec:introduction}

In reinforcement learning (RL)~\cite{sutton1998reinforcement}, target updates~\citep{MnihEtAl2015,LillicrapEtAl2015DDPG} are among the most common
stabilization devices for bootstrapping. In deep Q-learning~\citep{MnihEtAl2015}, a target network is
usually copied from the online network only periodically; in actor-critic
methods~\citep{LillicrapEtAl2015DDPG}, target networks often track the online networks through soft averaging.
These periodic hard updates and soft target updates are strongly supported by
empirical practice. However, their precise theoretical role is still not fully explained, especially when the
bootstrapped update is combined with function approximation.

Existing theory gives several important but different explanations. Target-based
temporal-difference (TD) learning~\citep{leehe2019target} analyzes separate online and target variables
for policy evaluation with linear function approximation (LFA). Periodic Q-learning~\citep{leehe2020periodic}
gives finite-sample guarantees for tabular control tasks.
Other analyses obtain stability through target networks together with
projections, truncation, regularization, nonlinear regularity assumptions, or
over-parameterized structure~\citep{zhang2021breaking,chen2023target,fellows2023why,che2024target}.
These results clarify significant parts of the target-network phenomenon, but
they do not prove the specific exact-convergence statement considered here. In this paper, the
statement is conditional: for the plain target-update recursion with a fixed
online step-size, under an explicit step-size range and explicit spectral
conditions, there are target-period or target-gain regimes in which the iterates
converge to the exact projected Q-Bellman solution, with the stochastic analogue obtained by applying the same mean-recursion
certificate after adding unbiased martingale-difference sampling errors under
standard stochastic-approximation
conditions~\citep{jaakkola1994convergence,tsitsiklis1994asynchronous,borkarmeyn2000ode}. Existing guarantees either concern tabular or policy-evaluation settings, add projections, truncation,
regularization, or structural assumptions to the algorithm, or leave
approximation, sampling, or regularization terms in the guarantee rather than an
exact convergence statement for the unmodified deterministic recursion.

The main contribution of this paper is this conditional exact convergence result
for target-updated linear Q-learning at a constant step-size. The period-$m$
hard-target method, denoted $m$-DLQL, is an interpolation between deterministic
linear Q-learning (DLQL) at $m=1$ and projected Q-value iteration (PQVI) as
$m\to\infty$. For hard
periodic targets, under the standing step-size range, the period-$m$ boundary map shows
that the relevant spectral condition at the PQVI endpoint certifies all
sufficiently large target periods; the corresponding DLQL endpoint condition
certifies the period-one side. For soft targets, the PQVI endpoint condition
certifies all sufficiently small positive target gains for both natural
orderings, and in the after-online-update ordering the DLQL endpoint condition
also certifies all gains sufficiently close to one. In every certified regime,
the error recursion converges to zero, so the limiting point is the exact
projected Q-Bellman solution rather than a neighborhood containing an additional
residual error term. We use switching systems~\citep{liberzon2003switching,lin2009stability,shorten2007stability} and the joint spectral radius (JSR)~\citep{rota1960note,blondel2005computational,jungers2009joint} as the main tools. We show that the error recursions of PQVI and DLQL become switching linear systems by following the recent switching-system approach to Q-learning~\citep{leehuhe2023discrete,lee2026lyapunovcertified,LeeLim2026}, and a JSR bound below one is the spectral certificate that gives uniform exponential convergence in the regimes described above.

The analysis begins with PQVI and one-period DLQL. These two methods solve the same
projected Q-Bellman equation, but their switching families are different, so one
spectral certificate need not imply the other. Periodic hard target updates then
connect these two endpoints. For the period-$m$ method, the
correct object is the target-boundary error map over one whole block of $m$
online updates. The period-one endpoint is DLQL, while the large-period endpoint
recovers the PQVI family under the frozen-target step-size condition. Therefore, when
PQVI is JSR-stable, all sufficiently large periods inherit a convergence
certificate, whereas intermediate periods are not universally ordered by the two
endpoints.

The soft-target part treats two natural orderings. We call the after-online-update
method SDLQL1 and the before-online-update method SDLQL2. In SDLQL1 the target is
averaged after the online update, while in SDLQL2 it is averaged before the
online update. In both cases the online and target errors form an augmented
switching linear system. The soft target gain is an independent parameter, not merely a rescaling
of the online step-size. For sufficiently small positive target gain, PQVI
stability yields augmented stability under the standing step-size condition. For
SDLQL1, the large-gain endpoint also reduces in JSR to DLQL, so DLQL stability
implies stability for gains sufficiently close to one; SDLQL2 has a delayed
endpoint and therefore requires its own augmented certificate.

The main text focuses on the deterministic version because it exposes the
feedback path created by the target update without sampling noise. Once the conditional deterministic spectral certificate is available, it also
identifies the mean recursion used when deterministic policy modes are replaced
by sampled stochastic modes and the standard stochastic-noise analysis is added,
as in recent switching-system treatments of Q-learning~\citep{lee2026lyapunovcertified,LeeLim2026}.

\section{Related Works}
Target networks entered modern deep reinforcement-learning practice through
hard-copy target updates in deep Q-learning and soft target averaging in
actor-critic methods~\citep{MnihEtAl2015,LillicrapEtAl2015DDPG}. In deep Q-networks (DQN)~\citep{MnihEtAl2015}, the target network is copied
periodically from the online network, so the target used in the Bellman update is
held fixed for multiple online steps. In deep deterministic
policy gradient (DDPG)~\citep{LillicrapEtAl2015DDPG}, the target networks track the online networks through soft averaging, producing the now-standard Polyak-type target update used in many
actor-critic algorithms. These mechanisms have strong empirical support, but their exact stabilizing effect is not fully explained by the usual constant-step-size linear Q-learning theory.

The closest theoretical literature analyzes target networks from several
complementary viewpoints. \citet{leehe2019target} study target-based TD learning
with separate online and target variables and analyze averaging, double, and
periodic target updates for policy evaluation with linear function approximation;
that work is about TD learning rather than the exact control-oriented linear
Q-learning switching family studied here. \citet{leehe2020periodic} analyze
periodic Q-learning in the tabular discounted setting and obtain a finite-sample
explanation of periodic targets, but the tabular model does not cover linear
function approximation or the JSR stability of the induced LFA error dynamics.
\citet{zhang2021breaking} give convergent target-network algorithms for the
deadly triad setting, including linear Q-learning, but the algorithms use
additional projections and regularization, so the limiting object is a
regularized fixed point rather than the exact projected Q-Bellman solution of the
plain constant-step-size recursion. \citet{chen2023target} prove finite-sample
stability guarantees for linear Q-learning using a
target network together with truncation; the truncation mechanism is essential in
their algorithmic design and leaves a sample-complexity guarantee up to function
approximation error, rather than an exact JSR convergence statement for the
untruncated deterministic target-update dynamics.

Other target-network analyses emphasize TD methods, nonlinear approximation, or
over-parameterization. \citet{fellows2023why} explain why target networks
stabilize temporal-difference methods through a partially fitted policy
evaluation viewpoint and allow nonlinear function approximation under regularity
and conditioning assumptions. \citet{che2024target} show that target
networks together with over-parameterized linear function approximation can
stabilize off-policy bootstrapping, primarily for value estimation and with
additional structural assumptions.

Therefore, existing theory explains important aspects of target networks, but it does
not show the specific statement proved here: under a constant online step-size,
if the relevant spectral conditions hold, then a sufficiently large hard-target
period or a sufficiently small positive soft-target gain yields error-free
convergence to the exact projected Q-Bellman solution. Some previous analyses use
the intuition that increasing the target period makes the update closer to a
projected or fitted value-iteration map, but they do not provide the same
zero-residual convergence statement for the plain period-$m$ linear Q-learning
recursion for all sufficiently large $m$. The present paper fills this gap by
using the hard-target boundary map, the soft-target augmented dynamics, and
finite-family JSR certificates as proof tools for this exact convergence claim.

\section{Preliminaries and Linear-Approximation Setup}
\label{sec:preliminaries}
\subsection{Notation}
\label{sec:notation}

The set of real numbers is denoted by $\R$; $\R^m$ is the $m$-dimensional
Euclidean space; and $\R^{m\times n}$ is the set of all $m\times n$ real
matrices. For a matrix $A$, $A^\top$ denotes its transpose. The identity
matrix is denoted by $I$. For vectors, $e_i$ is the $i$th standard basis vector,
with dimension clear from context, and $\otimes$ denotes the Kronecker product.
For a finite set $\mathcal S$, $|\mathcal S|$ denotes its cardinality. We write
\begin{align*}
\Delta_m:=\left\{q\in\R^m:q_i\geq0,\ \sum_{i=1}^m q_i=1\right\}
\end{align*}
for the probability simplex in $\R^m$. For a finite matrix family
$\mathcal H=\{A_1,\ldots,A_N\}$,
\begin{align*}
\co(\mathcal H)
:=\left\{\sum_{i=1}^N\lambda_i A_i:
\lambda_i\geq0,\ \sum_{i=1}^N\lambda_i=1\right\}
\end{align*}
denotes its convex hull.
We also use standard matrix notation that appears repeatedly below. For a
vector $x$, $\|x\|_2$ is the Euclidean norm. For a square matrix $A$,
$\rho(A)$ denotes its ordinary spectral radius, while $\rho(\mathcal H)$ denotes
the JSR of a switching family once the JSR is defined below.
For a matrix $B$, $\operatorname{range}(B)$ denotes its column space,
$B\succ0$ means that $B$ is symmetric positive definite, and
$\lambda_{\max}(B)$ denotes the largest eigenvalue when $B$ is symmetric.
Expectations are denoted by $\mathbb E[\cdot]$.

\subsection{Switched Linear Systems}
\label{sec:switching_systems}

The stability certificates used later are stated in the language of switched
systems, so we first recall the basic model before specializing it to the
Bellman-induced switching families. Let us consider the discrete-time switching
affine system (SAS)~\citep{liberzon2003switching,lin2009stability,shorten2007stability}
\begin{align*}
  x_{k+1}=A_{\sigma_k}x_k+b_{\sigma_k},
\end{align*}
where each index $i\in\{1,2,\ldots,M\}$, equivalently each affine pair
$(A_i,b_i)$, is called a \emph{mode}, and $\sigma_k$ is the switching signal that
selects the active mode at time $k$. The matrix $A_{\sigma_k}$ is selected from
the prescribed family $\mathcal H:=\{A_1,A_2,\ldots,A_M\}$, which is called a
\emph{switching family}; $b_{\sigma_k}$ is a mode-dependent affine term. When
$b_{\sigma_k}=0$, the deterministic part reduces to a linear switching system
(SLS), $x_{k+1}=A_{\sigma_k}x_k$. The worst-case exponential rate of the SLS
family is characterized by the JSR, defined as follows.
\begin{definition}
\label{def:jsr}
For a bounded set of matrices $\mathcal H\subset\R^{m\times m}$, its joint
spectral radius is
\begin{align*}
\jsr(\mathcal H)
:=
\lim_{k\to\infty}
\sup_{A_1,\ldots,A_k\in\mathcal H}
\|A_k\cdots A_1\|^{1/k}.
\end{align*}
\end{definition}
We note that the JSR is independent of the chosen submultiplicative
norm~\citep{rota1960note,jungers2009joint}. When $\mathcal H$ is finite, the
supremum for each fixed product length is a maximum over products generated by
matrices in $\mathcal H$. For a finite family $\mathcal H$, the notation
$\jsr(\co(\mathcal H))$ means the JSR computed when each factor in a product is
allowed to be any convex combination of matrices in $\mathcal H$.
Throughout the later JSR certificates, $\rho(\mathcal H)$ denotes this same JSR
value when the argument is a switching family.

\subsection{Joint Spectral Radius and Lyapunov Certificates}
\label{sec:joint_spectral_radius}

The JSR in~\Cref{def:jsr} turns arbitrary switched products into a
single worst-case exponential rate. An SLS is uniformly exponentially stable under arbitrary switching if there
exist constants $C\geq1$
and $\eta\in(0,1)$ such that $\|A_{\sigma_{k-1}}\cdots A_{\sigma_0}x\|_2\leq C\eta^k\|x\|_2$ for every horizon $k\geq0$, every initial state $x\in\R^m$, and every switching
sequence. A \emph{common Lyapunov function} for $\mathcal H$ is a positive definite
function that decreases along every mode. In the analysis below, the Bellman
maximum in linear Q-learning induces stochastic-policy switching, and the
Lyapunov functions are built from products of the corresponding mode matrices.
The following finite-family piecewise-quadratic construction~\citep{lee2026lyapunovcertified,hushenzhang2010generating} is the Lyapunov certificate used in the deterministic, stochastic, and target-network arguments.
It turns the abstract JSR condition into a concrete decrease estimate that can be
reused for each algorithmic family.
\begin{lemma}
\label{lem:common_lyapunov_construction}
Let $\mathcal H=\{A_1,A_2,\ldots,A_M\}\subset\R^{m\times m}$ and fix
$\epsilon>0$ such that $\beta_\epsilon:=\rho(\mathcal H)+\epsilon\in(0,1)$.
For a word $\sigma=(\sigma_1,\ldots,\sigma_k)\in\{1,\ldots,M\}^k$, write
\begin{align*}
A_\sigma:=A_{\sigma_k}\cdots A_{\sigma_1},
\end{align*}
with the convention that the empty word gives $A_\sigma=I$. Define
\begin{align*}
V_\epsilon^\infty(x)
:=\sum_{k=0}^\infty \beta_\epsilon^{-2k}
\max_{\sigma\in\{1,\ldots,M\}^k}\|A_\sigma x\|_2^2,
\qquad x\in\R^m.
\end{align*}
Then $V_\epsilon^\infty$ is finite for every $x$, and there exists
$C_\epsilon>0$ such that
\begin{align*}
\|x\|_2^2\leq V_\epsilon^\infty(x)\leq C_\epsilon\|x\|_2^2,
\qquad \forall x\in\R^m.
\end{align*}
The function $p_\epsilon(x):=\sqrt{V_\epsilon^\infty(x)}$ is a norm on
$\R^m$, and every mode satisfies
\begin{align*}
p_\epsilon(A_i x)\leq\beta_\epsilon p_\epsilon(x),
\qquad \forall x\in\R^m,
\qquad i=1,\ldots,M.
\end{align*}
\end{lemma}
\begin{proof}
The construction and proof are given in~\citep{lee2026lyapunovcertified,hushenzhang2010generating}; we
omit the proof here.
\end{proof}
Throughout the sequel, whenever this construction is applied to a switching
family with JSR less than one, we call the resulting $V_\epsilon^\infty$ a \emph{JSR
Lyapunov function} for that family, and we call the associated norm $p_\epsilon$
a \emph{JSR Lyapunov norm}.
The same Lyapunov argument will be used for PQVI, DLQL, hard targets, and soft
targets. The following lemma is the common bridge from a JSR bound to convergence of the corresponding error recursion.
\begin{lemma}
\label{lem:standard_jsr_convergence}
Let $\mathcal H=\{A_1,\ldots,A_M\}\subset\R^{n\times n}$ be finite and suppose
$\rho(\mathcal H)<1$. Let us consider any error recursion
\begin{align*}
x_{k+1}=A_kx_k,
\qquad A_k\in\co(\mathcal H),
\qquad k\in\{0,1,\ldots\}.
\end{align*}
Then, for every $\epsilon>0$ such that
$\beta_\epsilon:=\rho(\mathcal H)+\epsilon<1$, the Lyapunov function
$V_\epsilon^\infty$ and norm $p_\epsilon$ from~\Cref{lem:common_lyapunov_construction}, applied to
$\mathcal H$, satisfy
\begin{align*}
V_\epsilon^\infty(x_{k+1})
\leq
\beta_\epsilon^2 V_\epsilon^\infty(x_k),
\qquad
p_\epsilon(x_{k+1})\leq \beta_\epsilon p_\epsilon(x_k).
\end{align*}
Consequently, if $C_\epsilon$ is the constant from~\Cref{lem:common_lyapunov_construction}, then
\begin{align*}
p_\epsilon(x_k)\leq \beta_\epsilon^k p_\epsilon(x_0),
\qquad
\|x_k\|_2
\leq
\beta_\epsilon^k p_\epsilon(x_0)
\leq
\sqrt{C_\epsilon}\,\beta_\epsilon^k\|x_0\|_2,
\end{align*}
and hence $x_k\to0$.
\end{lemma}
\begin{proof}
By~\Cref{lem:common_lyapunov_construction}, $p_\epsilon(A_i x)\leq\beta_\epsilon p_\epsilon(x)$ for every $A_i\in\mathcal H$. If
$A_k=\sum_i\lambda_{k,i}A_i$ is a convex combination, then the triangle inequality
and homogeneity of the norm give $p_\epsilon(A_kx)
\leq
\sum_i\lambda_{k,i}p_\epsilon(A_i x)
\leq
\beta_\epsilon p_\epsilon(x)$. Applying this with $x=x_k$ gives the one-step norm contraction, and squaring it
gives the stated decrease of $V_\epsilon^\infty$. Iteration gives the
$p_\epsilon$ estimate, and the Euclidean bound follows from
$\|x\|_2\leq p_\epsilon(x)\leq\sqrt{C_\epsilon}\|x\|_2$.
\end{proof}

\subsection{Discounted Markov Decision Processes with Linear Function Approximation}
\label{sec:discounted_mdps_lfa}

We consider a finite discounted Markov decision process (MDP)~\cite{puterman2014markov,bertsekas1996neuro} with state-space
$\mathcal S=\{1,\ldots,|\mathcal S|\}$, action-space
$\mathcal A=\{1,\ldots,|\mathcal A|\}$, transition probability
$P(s'\mid s,a)$, real-valued one-step reward $r(s,a,s')$, expected reward $R(s,a):=\sum_{s'\in\mathcal S}P(s'\mid s,a)r(s,a,s')$, and discount factor $\gamma\in(0,1)$. State-action functions are viewed as
vectors in $\R^{|\mathcal S||\mathcal A|}$ using the action-block ordering $(1,1),(2,1),\ldots,(|\mathcal S|,1),
(1,2),(2,2),\ldots,(|\mathcal S|,|\mathcal A|)$. All matrices and vectors indexed by state-action pairs use this ordering. Define
\begin{align*}
P:=
\begin{bmatrix}
P_1\\
\vdots\\
P_{|\mathcal A|}
\end{bmatrix}
\in\R^{|\mathcal S||\mathcal A|\times |\mathcal S|},
\qquad
R:=
\begin{bmatrix}
R(\cdot,1)\\
\vdots\\
R(\cdot,|\mathcal A|)
\end{bmatrix}
\in\R^{|\mathcal S||\mathcal A|},
\end{align*}
where $P_a=P(\cdot\mid\cdot,a)\in\R^{|\mathcal S|\times|\mathcal S|}$.
Let $\Theta$ denote the set of deterministic stationary policies
$\pi:\mathcal S\to\mathcal A$. For any stochastic policy
$\mu:\mathcal S\to\Delta_{|\mathcal A|}$, define
\begin{align*}
\Pi^\mu:=
\begin{bmatrix}
\mu(1)^\top\otimes e_1^\top\\
\mu(2)^\top\otimes e_2^\top\\
\vdots\\
\mu(|\mathcal S|)^\top\otimes e_{|\mathcal S|}^\top
\end{bmatrix}
\in\R^{|\mathcal S|\times|\mathcal S||\mathcal A|}.
\end{align*}
For a deterministic policy $\pi\in\Theta$, the same notation $\Pi^\pi$ is used by
identifying $\pi(s)$ with its one-hot encoding.
For
$Q\in\R^{|\mathcal S||\mathcal A|}$, define
\begin{align*}
V_Q(s):=\max_{a\in\mathcal A}Q(s,a),
\qquad
V_Q:=(V_Q(1),\ldots,V_Q(|\mathcal S|))^\top.
\end{align*}
The Bellman optimality operator is $F(Q):=R+\gamma P V_Q$. Let $\Phi\in\R^{|\mathcal S||\mathcal A|\times m}$ be a feature matrix. Its
row corresponding to $(s,a)$ is $\phi(s,a)^\top$, where
$\phi(s,a)\in\R^m$. The LFA representation of the Q-function is $Q_\theta:=\Phi\theta$. For an LFA parameter $\theta$, define the corresponding greedy value vector by
\begin{align*}
V_\theta(s):=\max_{a\in\mathcal A}\phi(s,a)^\top\theta,
\qquad
V_\theta:=(V_\theta(1),\ldots,V_\theta(|\mathcal S|))^\top.
\end{align*}
We use $d$ to denote a state-action sampling distribution on
$\mathcal S\times\mathcal A$. In the i.i.d. observation model, $d$ is the
sampling distribution of $(s_k,a_k)$; in the Markovian observation model, $d$ is
the stationary state-action distribution of the behavior-induced chain. Define $D:=\diag(d(s,a))_{(s,a)\in\mathcal S\times\mathcal A}$. The feature and sampling conditions used throughout the rest of the paper are
collected in the following standing assumption. These conditions ensure that the
projected residual and the Gram matrix used later are well defined in the stated
coordinates.
\begin{assumption}
\label{ass:feature_sampling}
The feature matrix $\Phi$ has full column rank, and the sampling distribution
has full support: $d(s,a)>0$ for every
$(s,a)\in\mathcal S\times\mathcal A$. Equivalently, the diagonal sampling
matrix satisfies $D\succ0$.
\end{assumption}
This assumption also gives $\Phi^\top D\Phi\succ0$.
The switching-system model above will be used to analyze Bellman updates by
viewing each realization of the maximum operator as a policy-selected mode. To
make this passage from the Bellman maximum to a switching mode precise, we use
one final setup fact. It says that the difference between two greedy value
vectors is itself a value vector generated by a suitable stochastic policy
applied to the parameter difference. This is the device that converts the
Bellman maximum into a switching mode without changing the parameter space.
\begin{lemma}
\label{lem:stochastic_policy_linearization}
For every pair $\theta,\bar\theta\in\R^m$, there exists a stochastic policy
$\mu_{\theta,\bar\theta}$ such that
\begin{align}
V_\theta-V_{\bar\theta}
=
\Pi^{\mu_{\theta,\bar\theta}}\Phi(\theta-\bar\theta).
\label{eq:stoch_policy_linearization}
\end{align}
\end{lemma}
\begin{proof}
This is the same state-wise convex-interpolation argument used in the
switching-system proof of linear Q-learning in~\cite{lee2026lyapunovcertified};
the proof is omitted because only the notation changes here.
\end{proof}

\section{Projected Q-Value Iteration as the Fixed-Point Benchmark}
\label{sec:projected_qvi}

PQVI is the least-squares
fixed-point benchmark used in this section. By the Gram-matrix observation following~\Cref{ass:feature_sampling}, $\Phi^\top D\Phi$ is
nonsingular. Using the sampling matrix $D$ defined in~\Cref{sec:discounted_mdps_lfa},
the $D$-orthogonal projection onto $\operatorname{range}(\Phi)$ is $\Pi_D:=\Phi(\Phi^\top D\Phi)^{-1}\Phi^\top D$. The corresponding projected Bellman equation~\citep{bertsekas1996neuro,meyn2024projected,limlee2025understanding} is
\begin{align}
\Phi\theta
=\Pi_D F(\Phi\theta)
=\Pi_D\left(R+\gamma P V_\theta\right).
\label{eq:prelim_projected_bellman_equation}
\end{align}
A \emph{projected Bellman fixed point} is a parameter $\theta^\star$ satisfying
\begin{align*}
\Phi\theta^\star
=\Pi_D F(\Phi\theta^\star)
=\Pi_D\left(R+\gamma P V_{\theta^\star}\right).
\end{align*}
In general, such a fixed point need not exist, and when it exists it need not be
unique~\cite{limlee2025understanding}. In the convergence statements below,
$\theta^\star$ denotes the projected Bellman fixed point under consideration; the
JSR contraction assumptions used in the corresponding proofs identify uniqueness
for the associated update map.
Note also that the projected Bellman equation can be equivalently written as
\begin{align*}
\theta=(\Phi^\top D\Phi)^{-1}\Phi^\top D\left(R+\gamma P V_\theta\right),
\end{align*}
or equivalently the parameter form of the projected Q-Bellman equation,
\begin{align}
\Phi^\top D\Phi\,\theta=\Phi^\top D\left(R+\gamma P V_\theta\right).
\label{eq:projected_q_bellman_equation}
\end{align}

For a fixed target parameter $\bar\theta$, define the \emph{frozen-target
least-squares Bellman residual}
\begin{align*}
f(\theta ;\bar \theta )
:=
\left\| {R+\gamma P V_{\bar\theta}-\Phi\theta} \right\|_D^2,
\end{align*}
where $\|x\|_D^2:=x^\top D x$. With $\bar\theta$ fixed, minimizing
$f(\theta;\bar\theta)$ over the online parameter $\theta$ is a
$D$-least-squares projection problem. Its unique minimizer satisfies the normal
equation $\Phi^\top D\Phi\,\bar\theta^*
=\Phi^\top D(R+\gamma P V_{\bar\theta})$, and hence $\bar\theta^*
=(\Phi^\top D\Phi)^{-1}\Phi^\top D(R+\gamma P V_{\bar\theta})$. The fitted vector is equivalently the $D$-orthogonal projection of the frozen
Bellman target onto $\operatorname{range}(\Phi)$:
\begin{align*}
\Phi\bar\theta^*=
\Pi_D(R+\gamma P V_{\bar\theta}).
\end{align*}
Choosing $\bar\theta=\theta_k$ and writing the minimizer as $\theta_{k+1}$ gives
the PQVI update
\begin{align}
\theta_{k+1}
=
(\Phi^\top D\Phi)^{-1}\Phi^\top D
\left(R+\gamma P V_{\theta_k}\right),
\qquad
k\in\{0,1,\ldots\}.
\label{eq:pqvi_map_intro}
\end{align}
Therefore,~\Cref{eq:pqvi_map_intro} is the closed-form least-squares solution, and
its fitted value vector satisfies $\Phi\theta_{k+1}
=\Pi_D\left(R+\gamma P V_{\theta_k}\right)$.

To analyze convergence, we now pass to the SAS representation generated by the
Bellman maximum. For each PQVI iterate, choose a deterministic greedy
policy $\pi_k$ satisfying
\begin{align*}
\pi_k(s)\in\arg\max_{a\in\mathcal A}\phi(s,a)^\top\theta_k,
\qquad s\in\mathcal S,
\end{align*}
with fixed deterministic tie breaking. Then
$V_{\theta_k}=\Pi^{\pi_k}\Phi\theta_k$. This greedy choice first gives the affine
PQVI representation; after the projected Bellman fixed point is subtracted, the
common affine term cancels and the error dynamics become an SLS. To keep these
two elementary steps together, the next lemma derives the affine form inside its
proof and then gives the exact stochastic-policy error recursion. This prepares
the PQVI family that will serve as one endpoint in the target-update comparison.
\begin{lemma}
\label{lem:pqvi_stochastic_policy_error_recursion}
For each PQVI iterate, there exists a stochastic policy
$\mu_k=\mu_{\theta_k,\theta^\star}$ such that
\begin{align}
\theta_{k+1}-\theta^\star
=A_{\mu_k}^{\mathrm{PQVI}}(\theta_k-\theta^\star),\label{eq:1}
\end{align}
where for a stochastic policy $\mu:\mathcal S\to\Delta_{|\mathcal A|}$,
\begin{align}
A_\mu^{\mathrm{PQVI}}
:=
(\Phi^\top D\Phi)^{-1}\gamma\Phi^\top DP\Pi^\mu\Phi.
\label{eq:pqvi_mode_definition}
\end{align}
For a deterministic policy $\pi$, we write $A_\pi^{\mathrm{PQVI}}$.
\end{lemma}
\begin{proof}
First, by the greedy choice of $\pi_k$,
$V_{\theta_k}=\Pi^{\pi_k}\Phi\theta_k$. Substituting this identity into the
PQVI update~\Cref{eq:pqvi_map_intro} gives
\begin{align*}
\theta_{k+1}
=
(\Phi^\top D\Phi)^{-1}\Phi^\top D R
+
(\Phi^\top D\Phi)^{-1}\gamma\Phi^\top DP\Pi^{\pi_k}\Phi\,\theta_k,
\end{align*}
which is the greedy-policy affine representation. The projected Bellman
fixed-point equation for $\theta^\star$ gives $\theta^\star= (\Phi^\top D\Phi)^{-1}\Phi^\top D
\left(R+\gamma P V_{\theta^\star}\right)$.
Subtracting this identity from the PQVI update gives
\begin{align*}
\theta_{k+1}-\theta^\star
=
(\Phi^\top D\Phi)^{-1}\gamma\Phi^\top DP
\left(V_{\theta_k}-V_{\theta^\star}\right).
\end{align*}
By~\Cref{lem:stochastic_policy_linearization}, there exists a stochastic policy
$\mu_k$ such that $V_{\theta_k}-V_{\theta^\star}= \Pi^{\mu_k}\Phi(\theta_k-\theta^\star)$. Substitution and the definition of $A_{\mu_k}^{\mathrm{PQVI}}$ in~\Cref{eq:pqvi_mode_definition}
give $\theta_{k+1}-\theta^\star
=A_{\mu_k}^{\mathrm{PQVI}}(\theta_k-\theta^\star)$, which proves the claim.
\end{proof}
For deterministic policies, the finite PQVI switching family is
\begin{align}
\mathcal A^{\mathrm{PQVI}}
:=
\set{A_\pi^{\mathrm{PQVI}}:\pi\in\Theta}.
\label{eq:pqvi_family_intro}
\end{align}
The condition $\rho(\mathcal A^{\mathrm{PQVI}})<1$ ensures the stability of the switching model in~\Cref{eq:1} for PQVI. The implication $\rho(\mathcal A^{\mathrm{PQVI}})>1$ is a worst-case certificate failure; it does not by itself mean that a particular realized trajectory diverges. Some
realized trajectories can still converge, for example when the reachable modes
form a stable subset of the full switching family. Moreover, even if the JSR
of the full family is larger than one, convergence can still hold whenever the
modes that are actually reachable from the initial condition remain inside a
stable subfamily.

We can now present the PQVI convergence statement as the direct application of the
standard certificate. The lemma below makes explicit that the deterministic-policy
JSR condition gives global convergence of the nonlinear PQVI recursion because
every stochastic-policy difference mode remains inside the convex hull of the
finite deterministic PQVI family.
\begin{lemma}
\label{lem:pqvi_jsr_convergence}
If $\rho(\mathcal A^{\mathrm{PQVI}})<1$, then the PQVI map
in~\Cref{eq:pqvi_map_intro} has a unique fixed point $\theta^\star$. For every
initialization, the recursion in~\Cref{eq:pqvi_map_intro} converges to
$\theta^\star$. Moreover, for every $\epsilon>0$ such that
$\beta_\epsilon:=\rho(\mathcal A^{\mathrm{PQVI}})+\epsilon<1$, the finite-time
$p_\epsilon$ and Euclidean bounds in~\Cref{lem:standard_jsr_convergence} hold with
$x_k=\theta_k-\theta^\star$ and $\mathcal H=\mathcal A^{\mathrm{PQVI}}$. This fixed
point is equivalently the unique solution of the projected Q-Bellman equation
in~\Cref{eq:projected_q_bellman_equation}.
\end{lemma}
\begin{proof}
For two parameters $\theta$ and $\bar\theta$, subtracting their images under the
PQVI map in~\Cref{eq:pqvi_map_intro} gives
\begin{align*}
& (\Phi^\top D\Phi)^{-1}\Phi^\top D
\left(R+\gamma P V_\theta\right)
-(\Phi^\top D\Phi)^{-1}\Phi^\top D
\left(R+\gamma P V_{\bar\theta}\right)\\
&\qquad =
(\Phi^\top D\Phi)^{-1}\gamma\Phi^\top DP
\left(V_\theta-V_{\bar\theta}\right).
\end{align*}
By~\Cref{lem:stochastic_policy_linearization}, there exists a stochastic policy
$\mu$ such that
$V_\theta-V_{\bar\theta}=\Pi^\mu\Phi(\theta-\bar\theta)$. Therefore, using the
definition of $A_\mu^{\mathrm{PQVI}}$ in~\Cref{eq:pqvi_mode_definition}, the
difference of the two PQVI images is
\begin{align*}
A_\mu^{\mathrm{PQVI}}(\theta-\bar\theta).
\end{align*}
Since a stochastic policy is a state-wise convex combination of deterministic
policies and the PQVI mode is affine in $\Pi^\mu$, we have
$A_\mu^{\mathrm{PQVI}}\in\co(\mathcal A^{\mathrm{PQVI}})$. Applying
\Cref{lem:standard_jsr_convergence} to this one-step difference gives
\begin{align*}
p_\epsilon\left(A_\mu^{\mathrm{PQVI}}(\theta-\bar\theta)\right)
\leq
\beta_\epsilon p_\epsilon(\theta-\bar\theta),
\qquad \beta_\epsilon<1.
\end{align*}
Therefore, the PQVI map is a contraction in the norm $p_\epsilon$. Since the parameter
space is finite-dimensional and complete under this norm, the Banach fixed-point
theorem gives a unique fixed point $\theta^\star$~\citep{kreyszig1978introductory}. Taking $\bar\theta=\theta^\star$ and applying~\Cref{lem:standard_jsr_convergence} with
$x_k=\theta_k-\theta^\star$ and $\mathcal H=\mathcal A^{\mathrm{PQVI}}$ gives the
stated finite-time $p_\epsilon$ and Euclidean bounds, and hence
$\theta_k\to\theta^\star$. Finally, the fixed-point equation of
\Cref{eq:pqvi_map_intro} is exactly~\Cref{eq:projected_q_bellman_equation}, so the
limit is the unique solution of the projected Q-Bellman equation.
\end{proof}

\section{DLQL and Its Switching Certificate}
\label{sec:deterministic_linear_q_learning}

Given a transition sample $(s_k,a_k,r_{k+1},s_{k+1})$, where
$r_{k+1}=r(s_k,a_k,s_{k+1})$, the constant step-size linear Q-learning update~\citep{watkins1992q,tsitsiklis1994asynchronous,jaakkola1994convergence,sutton1998reinforcement} is
\begin{equation*}
\theta_{k+1}
=\theta_k+\alpha\phi(s_k,a_k)
\left(r_{k+1}+\gamma\max_{u\in\mathcal A}\phi(s_{k+1},u)^\top\theta_k
-\phi(s_k,a_k)^\top\theta_k\right),
\qquad
k\in\{0,1,\ldots\},
\end{equation*}
where $\alpha>0$ is the step-size. The parameter $\theta_k$ determines the
approximate action-value function $Q_{\theta_k}=\Phi\theta_k$, and the term
inside parentheses is the temporal-difference error formed from the greedy
next-action value under the same parameter. The deterministic version
of this update is
\begin{align}
\theta_{k+1}
=\theta_k+\alpha\Phi^\top D\left(R+\gamma P V_{\theta_k}-\Phi\theta_k\right),
\qquad
k\in\{0,1,\ldots\}.
\label{eq:prelim_deterministic_linear_q_learning}
\end{align}
Using the residual $f$ from~\Cref{sec:projected_qvi}, freeze the second argument
at $\bar\theta=\theta_k$. Then
\begin{align*}
\nabla_\theta f(\theta;\theta_k)
=
-2\Phi^\top D\left(R+\gamma P V_{\theta_k}-\Phi\theta\right),
\end{align*}
and hence~\Cref{eq:prelim_deterministic_linear_q_learning} is exactly the gradient
step
\begin{align*}
\theta_{k+1}
=
\theta_k-\frac{\alpha}{2}
\left.\nabla_\theta f(\theta;\theta_k)\right|_{\theta=\theta_k}.
\end{align*}
Therefore, deterministic linear Q-learning performs one residual-gradient step
with the Bellman target frozen at the current parameter for that step.
To proceed further, let us define the \emph{deterministic linear Q-learning map}
\begin{equation*}
h(\theta):=\theta+\alpha g(\theta),
\end{equation*}
where
\begin{equation*}
g(\theta):=\Phi^\top D\left(R+\gamma P V_\theta-\Phi\theta\right)
\end{equation*}
is the projected Bellman residual. With this notation,
~\Cref{eq:prelim_deterministic_linear_q_learning} is written compactly as
\begin{align*}
\theta_{k+1}=h(\theta_k),
\qquad
k\in\{0,1,\ldots\}.
\end{align*}
A fixed point of $h$ is equivalently a zero of $g$ and therefore
satisfies $\Phi^\top D\Phi\,\theta=\Phi^\top D\left(R+\gamma P V_\theta\right)$. This is the same parameter equation as the projected Q-Bellman equation
introduced above. The next lemma makes this equivalence explicit before the
DLQL switching family is derived.
\begin{lemma}
\label{lem:prelim_pbe_residual_equivalence}
The solution set of $g(\theta)=0$ is identical to the solution set of the
projected Bellman equation in~\Cref{eq:prelim_projected_bellman_equation}.
\end{lemma}
\begin{proof}
Note that $g(\theta)=0$ is $\Phi^\top D\Phi\,\theta=\Phi^\top D\left(R+\gamma P V_\theta\right)$, which is the parameter form of the projected Bellman equation because
$\Phi^\top D\Phi$ is nonsingular under~\Cref{ass:feature_sampling}.
\end{proof}
Therefore, DLQL and PQVI are the same at the
level of solution sets, which is the first point needed before comparing their
switching dynamics. The two algorithms differ in their error dynamics and therefore in their convergence
certificates.

Recent work develops an exact SLS model for the mean dynamics of linear Q-learning and uses the JSR as the stability certificate for products of switching modes~\citep{LeeLim2026}. We apply that same model to target-network variants in this paper.
To this end, for each iterate, choose a deterministic greedy policy $\pi_k$
satisfying
\begin{align*}
\pi_k(s)\in\arg\max_{a\in\mathcal A}\phi(s,a)^\top\theta_k,
\qquad s\in\mathcal S,
\end{align*}
with fixed deterministic tie breaking, so that
$V_{\theta_k}=\Pi^{\pi_k}\Phi\theta_k$. This greedy policy gives the SAS
form of DLQL, and the same calculation leads directly to the SLS error model
after subtracting $\theta^\star$. The next lemma combines these two steps and
identifies the DLQL family that will be compared with PQVI and target updates.
\begin{lemma}
\label{lem:direct_stochastic_policy_error_recursion}
For each DLQL iterate, there exists a stochastic policy
$\mu_k=\mu_{\theta_k,\theta^\star}$ such that the DLQL error recursion is
represented exactly as
\begin{align*}
\theta_{k+1}-\theta^\star
=A_{\mu_k}^{\mathrm{DLQL}}(\theta_k-\theta^\star),
\end{align*}
where for a stochastic policy $\mu:\mathcal S\to\Delta_{|\mathcal A|}$,
\begin{align}
A_\mu^{\mathrm{DLQL}}:=I-\alpha\Phi^\top D\Phi+\alpha\gamma\Phi^\top DP\Pi^\mu\Phi .
\label{eq:A_mu_definition}
\end{align}
For a deterministic policy $\pi$, we write $A_\pi^{\mathrm{DLQL}}$.
\end{lemma}
\begin{proof}
First, by the greedy choice of $\pi_k$,
$V_{\theta_k}=\Pi^{\pi_k}\Phi\theta_k$. Substituting this identity into~\Cref{eq:prelim_deterministic_linear_q_learning} gives
\begin{align*}
\theta_{k+1}
=\alpha\Phi^\top D R
+
\left(I-\alpha\Phi^\top D\Phi
+\alpha\gamma\Phi^\top DP\Pi^{\pi_k}\Phi\right)\theta_k,
\end{align*}
which is the greedy-policy SAS representation. Because $\theta^\star$ is the
projected Bellman fixed point, $g(\theta^\star)=0$. Subtracting the fixed-point
update from $h(\theta_k)$ and applying~\Cref{lem:stochastic_policy_linearization}
with $(\theta,\bar\theta)=(\theta_k,\theta^\star)$ gives a stochastic policy
$\mu_k$ satisfying $V_{\theta_k}-V_{\theta^\star}
= \Pi^{\mu_k}\Phi(\theta_k-\theta^\star)$.
Therefore, we have
\begin{align*}
\theta_{k+1}-\theta^\star
&=\theta_k-\theta^\star
-\alpha\Phi^\top D\Phi(\theta_k-\theta^\star)
+\alpha\gamma\Phi^\top DP\Pi^{\mu_k}\Phi(\theta_k-\theta^\star)\\
&=A_{\mu_k}^{\mathrm{DLQL}}(\theta_k-\theta^\star),
\end{align*}
which is the claimed switched error recursion.
\end{proof}
The corresponding finite deterministic switching family is
\begin{align}
\mathcal A^{\mathrm{DLQL}}
:=
\set{A_\pi^{\mathrm{DLQL}}: \pi\in\Theta}.
\label{eq:direct_family}
\end{align}
We now turn the DLQL switching family into a convergence statement. The next
lemma is the direct DLQL counterpart of the standard JSR certificate: a strict JSR
bound for the deterministic-policy family makes the nonlinear DLQL map a global
contraction after the Bellman maximum is linearized by a stochastic policy.
\begin{lemma}
\label{lem:dlql_jsr_convergence}
If $\rho(\mathcal A^{\mathrm{DLQL}})<1$, then the DLQL map $h$ has a unique fixed
point $\theta^\star$. For every initialization, the recursion
$\theta_{k+1}=h(\theta_k)$ converges to $\theta^\star$. Moreover, for every
$\epsilon>0$ such that
$\beta_\epsilon:=\rho(\mathcal A^{\mathrm{DLQL}})+\epsilon<1$, the finite-time
$p_\epsilon$ and Euclidean bounds in~\Cref{lem:standard_jsr_convergence} hold with
$x_k=\theta_k-\theta^\star$ and $\mathcal H=\mathcal A^{\mathrm{DLQL}}$. This fixed
point is equivalently the unique solution of the projected Bellman equation.
\end{lemma}
\begin{proof}
For two parameters $\theta$ and $\bar\theta$, the stochastic-policy linearization
in~\Cref{lem:stochastic_policy_linearization} gives
\begin{align*}
h(\theta)-h(\bar\theta)=A_\mu^{\mathrm{DLQL}}(\theta-\bar\theta)
\end{align*}
for some stochastic policy $\mu$. Since a stochastic policy is a state-wise convex
combination of deterministic policies and the DLQL mode is affine in $\Pi^\mu$,
we have $A_\mu^{\mathrm{DLQL}}\in\co(\mathcal A^{\mathrm{DLQL}})$. Applying
\Cref{lem:standard_jsr_convergence} to the difference of two DLQL trajectories
therefore gives
\begin{align*}
p_\epsilon(h(\theta)-h(\bar\theta))
\leq \beta_\epsilon p_\epsilon(\theta-\bar\theta),
\qquad \beta_\epsilon<1.
\end{align*}
Hence, $h$ is a contraction in the norm $p_\epsilon$, and since the parameter space
is finite-dimensional, $h$ has a unique fixed point $\theta^\star$. Taking
$\bar\theta=\theta^\star$ gives
\begin{align*}
\theta_{k+1}-\theta^\star=A_{\mu_k}^{\mathrm{DLQL}}(\theta_k-\theta^\star),
\qquad A_{\mu_k}^{\mathrm{DLQL}}\in\co(\mathcal A^{\mathrm{DLQL}}),
\end{align*}
so~\Cref{lem:standard_jsr_convergence} gives the stated finite-time bounds and
$\theta_k\to\theta^\star$. Finally,
\Cref{lem:prelim_pbe_residual_equivalence} identifies fixed points of $h$ with
solutions of the projected Bellman equation, so that solution is unique.
\end{proof}

Although DLQL and PQVI solve the same
projected Q-Bellman equation, convergence of one does not imply convergence of
the other. The DLQL update can be written as a residual step toward the
PQVI update evaluated at the same current parameter:
\begin{align*}
h(\theta_k)
=
\theta_k+
\alpha\Phi^\top D\Phi
\left(
(\Phi^\top D\Phi)^{-1}\Phi^\top D
\left(R+\gamma P V_{\theta_k}\right)
-\theta_k
\right).
\end{align*}
The two iterations coincide only in the special scalar-step-size case $\alpha\Phi^\top D\Phi=I$. Equality of fixed points therefore does not imply equality of convergence
behavior. The DLQL family is $\mathcal A^{\mathrm{DLQL}}$ from~\Cref{eq:direct_family}, whereas PQVI is governed by $\mathcal A^{\mathrm{PQVI}}$
from~\Cref{eq:pqvi_family_intro}. These families can have opposite stability
behavior. Such constructions are easy to find in scalar or small finite-state instances:
PQVI can converge while DLQL has a growing scalar iteration, and conversely DLQL
can reach the fixed point while PQVI has a growing scalar iteration.

\section{\texorpdfstring{$m$}{m}-DLQL: Periodic Hard Targets as Target-Boundary Error Maps}

The periodic hard target update used in deep Q-learning~\cite{MnihEtAl2015}
holds the target parameter fixed for several online updates and then copies the
online parameter into the target. Periodic target updates have also been studied
in tabular target Q-learning analyses~\citep{leehe2020periodic}. In the terminology
of this paper, the version with target-copy period $m$ is $m$-DLQL. This section first writes the corresponding Bellman operator and value-iteration form, and then passes to the exact SLS. A \emph{target boundary} is an update time at which the target parameter is reset or updated before a new target block begins. The \emph{boundary target} is the
value of $\bar\theta$ at that boundary; for a hard target, it is copied from
the online parameter and then held fixed during the following online updates.
The Bellman operator for one online update with frozen target $\bar\theta$ is
\begin{align*}
\theta\mapsto
\theta+
\alpha\Phi^\top D\left(R+\gamma P V_{\bar\theta}-\Phi\theta\right).
\end{align*}
Before introducing the error recursion around $\theta^\star$, it is useful to
write the online recursion itself. Starting from
$\theta_{t,0}:=\bar\theta_t=\theta_t$, a period with frozen target
$\bar\theta_t$ applies
\begin{align*}
\theta_{t,i+1}
=
\theta_{t,i}+
\alpha\Phi^\top D\left(R+\gamma P V_{\bar\theta_t}-\Phi\theta_{t,i}\right),
\qquad i\in\{0,1,2,\ldots,m-1\},
\end{align*}
and then copies $\bar\theta_{t+m}=\theta_{t,m}$. Let $\pi_t$ be a deterministic
greedy policy for the frozen target, so that
$V_{\bar\theta_t}=\Pi^{\pi_t}\Phi\bar\theta_t$. Then, the same within-period update has
the exact policy-selected form
\begin{align}
\theta_{t,i+1}
=
\theta_{t,i}+
\alpha\Phi^\top D\left(R+\gamma P\Pi^{\pi_t}\Phi\bar\theta_t-\Phi\theta_{t,i}\right),
\qquad i\in\{0,1,2,\ldots,m-1\}.
\label{eq:hard_target_policy_selected_update}
\end{align}
Since $\bar\theta_t$ is fixed throughout the period, each online update is the
gradient step
\begin{align*}
\theta_{t,i+1}
=
\theta_{t,i}-\frac{\alpha}{2}
\left.\nabla_\theta f(\theta;\bar\theta_t)\right|_{\theta=\theta_{t,i}},
\qquad i\in\{0,1,2,\ldots,m-1\}.
\end{align*}
As a result, the period-$m$ hard-target boundary output is obtained by taking $m$
gradient steps on $f(\cdot;\bar\theta_t)$ while the target argument
$\bar\theta_t$ is frozen.
With this interpretation in place, the $m$-DLQL update is summarized in~\Cref{alg:period_m_hard_target_vi}.
\begin{algorithm}[t]
\caption{$m$-DLQL: Period-$m$ Hard-Target Linear Q-learning}
\label{alg:period_m_hard_target_vi}
\begin{algorithmic}[1]
\REQUIRE Step-size $\alpha>0$, period $m\geq1$, and initial parameter $\theta_0$.
\FOR{target boundaries $t=0,m,2m,\ldots$}
\STATE Copy the online parameter into the target: $\bar\theta_t=\theta_t$.
\STATE Freeze $\bar\theta_t$ during the following $m$ online updates.
\FOR{$i\in\{0,1,2,\ldots,m-1\}$}
\STATE Update
$\displaystyle
\theta_{t+i+1}
=
\theta_{t+i}+
\alpha\Phi^\top D\left(R+\gamma P V_{\bar\theta_t}-\Phi\theta_{t+i}\right).
$
\ENDFOR
\STATE Copy the online parameter into the target at the next boundary:
$\bar\theta_{t+m}=\theta_{t+m}$.
\ENDFOR
\end{algorithmic}
\end{algorithm}

\subsection{SLS Representation}

Let $m\geq 1$ be the target-copy period. At the beginning of a period, assume
that the online and target parameters agree at the period boundary:
$\bar\theta_t=\theta_t$. Equivalently, their errors around the projected Bellman
fixed point agree: $\bar\theta_t-\theta^\star=\theta_t-\theta^\star$.
During the following $m$ online updates, the target is frozen. Since the frozen
target fixes the Bellman maximum over the whole block, one policy-selected mode
is held over that period. The next lemma gives the boundary error map and, inside
the statement, defines the period-$m$ stochastic mode needed for that map.
\begin{lemma}
\label{lem:hard_period_boundary_map}
For a stochastic policy $\mu:\mathcal S\to\Delta_{|\mathcal A|}$ and period
length $m\geq1$, define
\begin{align*}
A_{\mu,m}^{\mathrm{DLQL}}
:={}&(I-\alpha\Phi^\top D\Phi)^m
+\sum_{i=0}^{m-1}(I-\alpha\Phi^\top D\Phi)^i
\alpha\gamma\Phi^\top DP\Pi^\mu\Phi .
\end{align*}
If the target is frozen at $\bar\theta_t$ during the period and
$\bar\theta_t=\theta_t$ at the target boundary, then after the hard copy
$\bar\theta_{t+m}=\theta_{t+m}$ there exists a stochastic policy
$\mu_t=\mu_{\bar\theta_t,\theta^\star}$, fixed throughout the period, such that
\begin{align}
\theta_{t+m}-\theta^\star=A_{\mu_t,m}^{\mathrm{DLQL}}(\theta_t-\theta^\star),
\qquad
\bar\theta_{t+m}-\theta^\star=\theta_{t+m}-\theta^\star.
\label{eq:hard_stochastic_period_map}
\end{align}
For a deterministic policy $\pi\in\Theta$, we write
$A_{\pi,m}^{\mathrm{DLQL}}$ for the same period map with $\mu=\pi$.
\end{lemma}
\begin{proof}
The projected Q-Bellman equation for $\theta^\star$ is equivalent to $0=
\alpha\Phi^\top D\left(R+\gamma P V_{\theta^\star}-\Phi\theta^\star\right)$.
Subtracting this identity from the frozen-target update~\Cref{eq:hard_target_policy_selected_update} gives, for each
$i\in\{0,1,2,\ldots,m-1\}$,
\begin{align*}
\theta_{t+i+1}-\theta^\star
&=
\theta_{t+i}-\theta^\star
+
\alpha\Phi^\top D\left(\gamma P(V_{\bar\theta_t}-V_{\theta^\star})
-\Phi(\theta_{t+i}-\theta^\star)\right).
\end{align*}
By~\Cref{lem:stochastic_policy_linearization}, there exists a stochastic policy
$\mu_t=\mu_{\bar\theta_t,\theta^\star}$ such that
$V_{\bar\theta_t}-V_{\theta^\star}
=\Pi^{\mu_t}\Phi(\bar\theta_t-\theta^\star)$. Because $\bar\theta_t$ is frozen,
this same $\mu_t$ is fixed throughout the period. Hence the one-step recursion is
\begin{align}
\theta_{t+i+1}-\theta^\star
=(I-\alpha\Phi^\top D\Phi)(\theta_{t+i}-\theta^\star)
+\alpha\gamma\Phi^\top DP\Pi^{\mu_t}\Phi\,(\bar\theta_t-\theta^\star),
\qquad i\in\{0,1,2,\ldots,m-1\}.
\label{eq:hard_within_period}
\end{align}
Summing this recursion over the period gives, by induction on $j=1,\ldots,m$,
\begin{align*}
\theta_{t+j}-\theta^\star
&=
(I-\alpha\Phi^\top D\Phi)^j(\theta_t-\theta^\star)
+\sum_{i=0}^{j-1}(I-\alpha\Phi^\top D\Phi)^i
\alpha\gamma\Phi^\top DP\Pi^{\mu_t}\Phi\,(\bar\theta_t-\theta^\star).
\end{align*}
The case $j=1$ is exactly~\Cref{eq:hard_within_period}; the induction step is
obtained by applying~\Cref{eq:hard_within_period} once more and shifting the
finite sum. Using $\bar\theta_t-\theta^\star=\theta_t-\theta^\star$ at $j=m$
gives the first identity in~\Cref{eq:hard_stochastic_period_map}. The hard copy at
the end of the period gives $\bar\theta_{t+m}=\theta_{t+m}$, and subtracting
$\theta^\star$ gives the second identity.
\end{proof}
For deterministic policies, the admissible target-boundary modes are the matrices
$A_{\pi,m}^{\mathrm{DLQL}}$, and arbitrary products of these modes describe
arbitrary deterministic switching across target periods. Define the corresponding
finite switching family
\begin{align*}
\mathcal A_m^{\mathrm{DLQL}} := \set{A_{\pi,m}^{\mathrm{DLQL}}:\pi\in\Theta}.
\end{align*}
Consequently, the period-$m$ hard-target update has an SLS model at target
boundaries. The number $\rho(\mathcal A_m^{\mathrm{DLQL}})$ is the JSR of the
matrices $A_{\pi,m}^{\mathrm{DLQL}}$ that map one target-boundary error to the next
after $m$ online updates. The next lemma turns this boundary SLS description into
the hard-target convergence certificate used in the endpoint comparison below.
\begin{lemma}
\label{lem:hard_target_boundary_jsr_convergence}
Let $m\geq1$ be fixed and let $\theta^\star$ be the projected Bellman fixed point
under consideration. If $\rho(\mathcal A_m^{\mathrm{DLQL}})<1$, then from every initial parameter the period-$m$
hard-target update converges to $\theta^\star$ at target boundaries and along the
within-period online iterates. Moreover, for every $\epsilon>0$ such that
$\beta_\epsilon:=\rho(\mathcal A_m^{\mathrm{DLQL}})+\epsilon<1$, the finite-time
$p_\epsilon$ and Euclidean bounds in~\Cref{lem:standard_jsr_convergence} hold with
$x_k=\theta_{km}-\theta^\star$ and $\mathcal H=\mathcal A_m^{\mathrm{DLQL}}$.
Since one target period contains $m$ online updates, the corresponding normalized
online-step JSR rate is $\rho(\mathcal A_m^{\mathrm{DLQL}})^{1/m}$.
\end{lemma}
\begin{proof}
At each target boundary $t=km$, the hard copy gives $\bar\theta_{km}=\theta_{km}$.
By~\Cref{lem:hard_period_boundary_map}, there exists a stochastic policy
$\mu_{km}$, fixed during that period, such that
\begin{align*}
\theta_{(k+1)m}-\theta^\star
=A_{\mu_{km},m}^{\mathrm{DLQL}}(\theta_{km}-\theta^\star).
\end{align*}
A stochastic policy is a state-wise convex combination of deterministic policies,
and $A_{\mu_{km},m}^{\mathrm{DLQL}}$ is affine in $\Pi^{\mu_{km}}$. Therefore
$A_{\mu_{km},m}^{\mathrm{DLQL}}\in\co(\mathcal A_m^{\mathrm{DLQL}})$. Applying
\Cref{lem:standard_jsr_convergence} with $x_k=\theta_{km}-\theta^\star$ and
$\mathcal H=\mathcal A_m^{\mathrm{DLQL}}$ gives the stated finite-time bounds and
$\theta_{km}\to\theta^\star$.

It remains only to pass from boundary convergence to convergence of all online
iterates. For each fixed within-period index $i\in\{0,1,2,\ldots,m-1\}$, the finite
iteration formula in the proof of~\Cref{lem:hard_period_boundary_map} gives
\begin{align*}
\theta_{km+i}-\theta^\star
={}&(I-\alpha\Phi^\top D\Phi)^i(\theta_{km}-\theta^\star)\\
&+\sum_{j=0}^{i-1}(I-\alpha\Phi^\top D\Phi)^j
\alpha\gamma\Phi^\top DP\Pi^{\mu_{km}}\Phi(\theta_{km}-\theta^\star).
\end{align*}
For fixed $m$, the matrices in this finite expression are uniformly bounded over
all stochastic policies. Hence $\theta_{km+i}-\theta^\star\to0$ for every
$i\in\{0,1,2,\ldots,m-1\}$, and all within-period online iterates converge to
$\theta^\star$ as well.
\end{proof}

\subsection{Interpolation Interpretation}

We next explain how the interpretation of $m$-DLQL changes as the target period
$m$ varies. The key algebraic fact is that the deterministic period map separates
into a PQVI component and a frozen-target residual that vanishes when the matrix
$I-\alpha\Phi^\top D\Phi$ is Schur-stable. The following lemma states this identity, the
period-one endpoint, and the large-period JSR limit before these facts are used
for the interpolation picture.
\begin{lemma}
\label{lem:m_dlql_jsr_convergence}
For every deterministic policy $\pi\in\Theta$,
\begin{align*}
A_{\pi,1}^{\mathrm{DLQL}}=A_\pi^{\mathrm{DLQL}},
\qquad
\rho(\mathcal A_1^{\mathrm{DLQL}})=\rho(\mathcal A^{\mathrm{DLQL}}).
\end{align*}
If $\rho(I-\alpha\Phi^\top D\Phi)<1$, then
\begin{align*}
\rho(\mathcal A_m^{\mathrm{DLQL}})\to\rho(\mathcal A^{\mathrm{PQVI}})
\qquad\text{as }m\to\infty.
\end{align*}
\end{lemma}
\begin{proof}
For each deterministic policy $\pi\in\Theta$, the specialization of the period
map in~\Cref{lem:hard_period_boundary_map} gives
\begin{align*}
A_{\pi,m}^{\mathrm{DLQL}}
&=(I-\alpha\Phi^\top D\Phi)^m
+\sum_{i=0}^{m-1}(I-\alpha\Phi^\top D\Phi)^i
\alpha\gamma\Phi^\top DP\Pi^\pi\Phi.
\end{align*}
Using the definition of $A_\pi^{\mathrm{PQVI}}$ in~\Cref{eq:pqvi_mode_definition},
\begin{align*}
\alpha\gamma\Phi^\top DP\Pi^\pi\Phi
&=\alpha\Phi^\top D\Phi
(\Phi^\top D\Phi)^{-1}\gamma\Phi^\top DP\Pi^\pi\Phi\\
&=\alpha\Phi^\top D\Phi A_\pi^{\mathrm{PQVI}}.
\end{align*}
Since $\alpha\Phi^\top D\Phi$ commutes with $I-\alpha\Phi^\top D\Phi$, we have
\begin{align*}
\sum_{i=0}^{m-1}(I-\alpha\Phi^\top D\Phi)^i\alpha\Phi^\top D\Phi
&=\sum_{i=0}^{m-1}(I-\alpha\Phi^\top D\Phi)^i
\left[I-(I-\alpha\Phi^\top D\Phi)\right]\\
&=I-(I-\alpha\Phi^\top D\Phi)^m.
\end{align*}
Therefore, it follows that 
\begin{align}
A_{\pi,m}^{\mathrm{DLQL}}
&=(I-\alpha\Phi^\top D\Phi)^m
+\left[I-(I-\alpha\Phi^\top D\Phi)^m\right]A_\pi^{\mathrm{PQVI}} \nonumber\\
&=A_\pi^{\mathrm{PQVI}}+(I-\alpha\Phi^\top D\Phi)^m(I-A_\pi^{\mathrm{PQVI}}).
\label{eq:m_dlql_converges_to_pqvi}
\end{align}
Setting $m=1$ in~\Cref{eq:m_dlql_converges_to_pqvi} gives
\begin{align*}
A_{\pi,1}^{\mathrm{DLQL}}
&=A_\pi^{\mathrm{PQVI}}+(I-\alpha\Phi^\top D\Phi)(I-A_\pi^{\mathrm{PQVI}})\\
&=I-\alpha\Phi^\top D\Phi+
\alpha\Phi^\top D\Phi A_\pi^{\mathrm{PQVI}}\\
&=I-\alpha\Phi^\top D\Phi+
\alpha\gamma\Phi^\top DP\Pi^\pi\Phi
=A_\pi^{\mathrm{DLQL}},
\end{align*}
where the last equality is~\Cref{eq:A_mu_definition} for deterministic $\pi$.
Therefore, the two period-one families coincide, and their JSRs are equal.

If $\rho(I-\alpha\Phi^\top D\Phi)<1$, then
$(I-\alpha\Phi^\top D\Phi)^m\to0$. Therefore~\Cref{eq:m_dlql_converges_to_pqvi}
gives $A_{\pi,m}^{\mathrm{DLQL}}\to A_\pi^{\mathrm{PQVI}}$ for each deterministic
policy $\pi$, uniformly over the finite policy set $\Theta$. The JSR of a finite
family is continuous with respect to the matrix entries~\citep{heilstrang1995continuity,jungers2009joint}.
Hence $\rho(\mathcal A_m^{\mathrm{DLQL}})\to\rho(\mathcal A^{\mathrm{PQVI}})$.
\end{proof}

The endpoint identity and large-period limit in
\Cref{lem:m_dlql_jsr_convergence} show that, under the step-size condition below,
the period length acts as a single interpolation parameter between the one-step
target update of DLQL and the fully PQVI endpoint, as illustrated in
\Cref{fig:hard_period_interpolation}.
\begin{figure}[H]
\centering
\begin{tikzpicture}[>=Stealth, line cap=round, line join=round]
  \draw[line width=1.35pt] (0,0) -- (10,0);
  \fill[red] (0,0) circle (2.8pt);
  \fill (5,0) circle (2.8pt);
  \fill[blue] (10,0) circle (2.8pt);

  \node[align=center, anchor=south west] at (-0.75,0.35)
    {DLQL};
  \node[align=center, anchor=south east] at (10.75,0.35)
    {PQVI};

  \draw[thick, ->] (5,1.18) -- (5,0.13);
  \node[align=center, anchor=south] at (5,1.18) {$m$-DLQL};

  \node[anchor=north] at (0,-0.30) {$m=1$};
  \node[anchor=north] at (5,-0.30) {$m$};
  \node[anchor=north] at (10,-0.30) {$m=\infty$};

  \draw[thick, ->] (3.65,-1.08) -- (6.35,-1.08);
  \node[anchor=north] at (5,-1.20) {$m\to\infty$};
\end{tikzpicture}
\caption{Hard-target period length as an interpolation parameter. The
period $m=1$ boundary mode is the DLQL mode,
while the infinite-period endpoint $m=\infty$ recovers PQVI.}
\label{fig:hard_period_interpolation}
\end{figure}
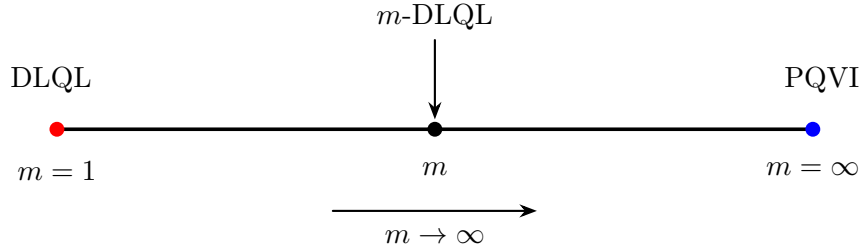

In particular, when the step-size $\alpha$ is sufficiently small,
\Cref{eq:m_dlql_converges_to_pqvi} separates the PQVI component
$A_\pi^{\mathrm{PQVI}}$ from the residual term
$(I-\alpha\Phi^\top D\Phi)^m(I-A_\pi^{\mathrm{PQVI}})$. At $m=1$, this is the DLQL
mode $A_\pi^{\mathrm{DLQL}}$; as $m\to\infty$, the residual term vanishes whenever
$\rho(I-\alpha\Phi^\top D\Phi)<1$, and the finite-family JSR converges to the
PQVI JSR by~\Cref{lem:m_dlql_jsr_convergence}. This identifies the large-period endpoint of the
hard-target boundary family. Because
$\Phi^\top D\Phi$ is symmetric positive definite by the Gram-matrix observation
following~\Cref{ass:feature_sampling}, a sufficiently small positive step-size
$\alpha$ makes the spectral radius of $I-\alpha\Phi^\top D\Phi$ smaller than one.
The next lemma gives the exact allowable range, which is later used whenever a
frozen-target online update must be stable.
\begin{lemma}
\label{lem:relaxation_stepsize_condition}
The matrix $I-\alpha\Phi^\top D\Phi$ satisfies
\begin{align}
\rho(I-\alpha\Phi^\top D\Phi)<1
\quad\Longleftrightarrow\quad
0<\alpha<\frac{2}{\lambda_{\max}(\Phi^\top D\Phi)}.
\label{eq:relaxation_stepsize_condition}
\end{align}
\end{lemma}
\begin{proof}
Since $\Phi^\top D\Phi$ is symmetric positive
definite, it has eigenvalues with $0<\lambda_{\min}(\Phi^\top D\Phi)
\leq \lambda_i(\Phi^\top D\Phi)
\leq \lambda_{\max}(\Phi^\top D\Phi)$.
The eigenvalues of $I-\alpha\Phi^\top D\Phi$ are
$1-\alpha\lambda_i(\Phi^\top D\Phi)$. Hence, this matrix has spectral radius smaller than one exactly
when $|1-\alpha\lambda_i(\Phi^\top D\Phi)|<1$ for every $i$, which is equivalent to $0<\alpha<2/\lambda_i(\Phi^\top D\Phi)$ for every $i$,
and therefore to~\Cref{eq:relaxation_stepsize_condition}.
\end{proof}
We use this range as the standing step-size condition throughout the paper.
\begin{assumption}
\label{ass:hard_target_stepsize}
The step-size satisfies~\Cref{eq:relaxation_stepsize_condition}, namely $0<\alpha<\frac{2}{\lambda_{\max}(\Phi^\top D\Phi)}$.
\end{assumption}
Once the step-size condition holds, the stability of the large-period regime is
governed by the JSR of the PQVI family $\mathcal A^{\mathrm{PQVI}}$ in~\Cref{eq:pqvi_family_intro}. We note that PQVI does not universally dominate DLQL in JSR: $\mathcal A^{\mathrm{DLQL}}$ from~\Cref{eq:direct_family} and $\mathcal A^{\mathrm{PQVI}}$ are generally distinct switched
families, so the projection can either help or hurt. The comparison between
large and small target periods is therefore inherently conditional. If
$\mathcal A^{\mathrm{PQVI}}$ yields a strict JSR advantage, then the large-period hard-target
family inherits this advantage for the target-boundary recursion; in particular, whenever
$\rho(\mathcal A^{\mathrm{PQVI}})<1$, we have $\rho(\mathcal A_m^{\mathrm{DLQL}})<1$ for all
sufficiently large $m$, and $m$-DLQL converges by~\Cref{lem:hard_target_boundary_jsr_convergence}.

We therefore compare the two period endpoints: the large-period endpoint
represented by $\mathcal A^{\mathrm{PQVI}}$, and the small-period endpoint represented by
$\mathcal A^{\mathrm{DLQL}}$. The four possible endpoint stability configurations are
\begin{align*}
\begin{array}{ll}
\text{(i)} & \rho(\mathcal A^{\mathrm{PQVI}})<1 \ \text{and}\ \rho(\mathcal A^{\mathrm{DLQL}})<1,\\
\text{(ii)} & \rho(\mathcal A^{\mathrm{PQVI}})<1<\rho(\mathcal A^{\mathrm{DLQL}}),\\
\text{(iii)} & \rho(\mathcal A^{\mathrm{DLQL}})<1<\rho(\mathcal A^{\mathrm{PQVI}}),\\
\text{(iv)} & \rho(\mathcal A^{\mathrm{PQVI}})>1 \ \text{and}\ \rho(\mathcal A^{\mathrm{DLQL}})>1.
\end{array}
\end{align*}
\begin{theorem}
\label{thm:hard_endpoint_regimes}
Suppose~\Cref{ass:hard_target_stepsize} holds. The following endpoint regimes hold.
\begin{enumerate}[label=(\roman*)]
\item If $\rho(\mathcal A^{\mathrm{PQVI}})<1$ and
$\rho(\mathcal A^{\mathrm{DLQL}})<1$, then there exists $m_0\geq1$ such that, for
every integer $1\leq m\leq m_0$,
$\rho(\mathcal A_m^{\mathrm{DLQL}})^{1/m}<1$. Moreover, there exists $m_1\geq1$
such that, for every $m\geq m_1$,
$\rho(\mathcal A_m^{\mathrm{DLQL}})^{1/m}<1$.
\item If $\rho(\mathcal A^{\mathrm{PQVI}})<1<\rho(\mathcal A^{\mathrm{DLQL}})$, then
there exists $m_0\geq1$ such that, for every integer $1\leq m\leq m_0$,
$\rho(\mathcal A_m^{\mathrm{DLQL}})>1$, and there exists $m_1\geq1$ such that,
for every $m\geq m_1$,
$\rho(\mathcal A_m^{\mathrm{DLQL}})^{1/m}<1<\rho(\mathcal A^{\mathrm{DLQL}})$.
\item If $\rho(\mathcal A^{\mathrm{DLQL}})<1<\rho(\mathcal A^{\mathrm{PQVI}})$, then
there exists $m_0\geq1$ such that, for every integer $1\leq m\leq m_0$,
$\rho(\mathcal A_m^{\mathrm{DLQL}})^{1/m}<1$, and there exists $m_1\geq1$ such
that, for every $m\geq m_1$,
$\rho(\mathcal A^{\mathrm{DLQL}})<1<\rho(\mathcal A_m^{\mathrm{DLQL}})$.
\item If $\rho(\mathcal A^{\mathrm{PQVI}})>1$ and
$\rho(\mathcal A^{\mathrm{DLQL}})>1$, then there exists $m_0\geq1$ such that, for
every integer $1\leq m\leq m_0$, $\rho(\mathcal A_m^{\mathrm{DLQL}})>1$, and
there exists $m_1\geq1$ such that, for every $m\geq m_1$,
$\rho(\mathcal A_m^{\mathrm{DLQL}})>1$.
\end{enumerate}
\end{theorem}
\begin{proof}
The period-one identity in~\Cref{lem:m_dlql_jsr_convergence} gives
$\rho(\mathcal A_1^{\mathrm{DLQL}})=\rho(\mathcal A^{\mathrm{DLQL}})$. Therefore, whenever
the DLQL endpoint is stable, the small-period stable claim holds by taking
$m_0=1$; whenever the DLQL endpoint is unstable, the small-period unstable claim
holds by taking $m_0=1$.

The large-period limit in~\Cref{lem:m_dlql_jsr_convergence} gives
$\rho(\mathcal A_m^{\mathrm{DLQL}})\to\rho(\mathcal A^{\mathrm{PQVI}})$. If the PQVI
endpoint has JSR below one, then $\rho(\mathcal A_m^{\mathrm{DLQL}})<1$ for all
sufficiently large $m$, and hence
$\rho(\mathcal A_m^{\mathrm{DLQL}})^{1/m}<1$ for all sufficiently large $m$. If the
PQVI endpoint has JSR above one, then $\rho(\mathcal A_m^{\mathrm{DLQL}})>1$ for
all sufficiently large $m$. Combining these two endpoint consequences with the
four possible strict endpoint configurations proves (i)--(iv).
\end{proof}
This combined endpoint theorem does not imply that every intermediate period is
certified when the endpoints are certified, and it does not rule out
intermediate-period stabilization when the endpoints fail. An intermediate $m$
can fail the JSR test even when both endpoints pass it, and conversely there can
be an $m$ between the small-period and large-period regimes for which the
hard-target family $\mathcal A_m^{\mathrm{DLQL}}$ is JSR-stable. Throughout, an
endpoint JSR above one means failure of the worst-case certificate, not divergence
of every realized trajectory.
These cases clarify the period choice. Small $m$ is favorable when the DLQL
family is already stable while the PQVI is unstable, as in
$\rho(\mathcal A^{\mathrm{DLQL}})<1<\rho(\mathcal A^{\mathrm{PQVI}})$. Large $m$ is favorable when the
PQVI is stable or has the strict JSR advantage, especially in
the case $\rho(\mathcal A^{\mathrm{PQVI}})<1<\rho(\mathcal A^{\mathrm{DLQL}})$. When both endpoints fail
the JSR stability test, only intermediate finite periods can possibly help, and
the endpoint arguments alone give no guarantee.
Finite small-state constructions can realize this mechanism, but they are omitted here.

The plain tabular comparison is simpler in the usual synchronous
identity-feature tabular setting. There, PQVI, one-period DLQL, and $m$-DLQL
reduce to scaled row-stochastic switching families. The update rules, switching systems, and JSR comparison are
collected in~\Cref{app:tabular_pqvi_vs_dlql}: PQVI has JSR $\gamma$, one-period
DLQL has JSR $1-\alpha(1-\gamma)$, and the period-$m$ hard-target family has JSR
$\gamma+(1-\gamma)(1-\alpha)^m$ for $0<\alpha\leq1$. Therefore, the two tabular
endpoint maps coincide at $\alpha=1$, while for $0<\alpha<1$ the $m$-periodic
rate decreases monotonically from the DLQL endpoint to the PQVI endpoint as
$m\to\infty$.

\section{SDLQL1: Soft Target Averaging After the Online Update}
\label{sec:soft_target_after_online_update}
After the hard-copy target mechanism, we now turn to \emph{soft DLQL with
target averaging after the online update}. Throughout this section, we use the
step-size condition in~\Cref{ass:hard_target_stepsize}. Soft target averaging is
widely used in actor-critic algorithms such as DDPG~\citep{LillicrapEtAl2015DDPG}. In this ordering,
the online parameter is updated first and the target then tracks the updated
online parameter. The next section treats the matching
before-online-update ordering. The after-online-update SDLQL1 version is
\begin{align*}
\theta_{k+1}
&=
\theta_k+
\alpha\Phi^\top D\left(R+\gamma P V_{\bar\theta_k}-\Phi\theta_k\right),\\
\bar\theta_{k+1}
&=
(1-\beta)\bar\theta_k+\beta\theta_{k+1}.
\end{align*}

Equivalently, the augmented update is represented by the mapping
\begin{align}
(\theta_k,\bar\theta_k)
\mapsto
\begin{bmatrix}
\theta_k+
\alpha\Phi^\top D(R+\gamma P V_{\bar\theta_k}-\Phi\theta_k)\\
(1-\beta)\bar\theta_k+
\beta\left[\theta_k+
\alpha\Phi^\top D(R+\gamma P V_{\bar\theta_k}-\Phi\theta_k)\right]
\end{bmatrix}.
\label{eq:soft_after_updated_online_aug_map}
\end{align}
Because $\theta^\star$ is the projected Bellman fixed point,
$(\theta^\star,\theta^\star)$ is a fixed point of~\Cref{eq:soft_after_updated_online_aug_map}.
By~\Cref{eq:stoch_policy_linearization}, for every $k$ there exists a stochastic
policy $\mu_k=\mu_{\bar\theta_k,\theta^\star}$ such that $V_{\bar\theta_k}-V_{\theta^\star} = \Pi^{\mu_k}\Phi(\bar\theta_k-\theta^\star)$. The online row is
\begin{equation*}
\theta_{k+1}-\theta^\star
=(I-\alpha\Phi^\top D\Phi)(\theta_k-\theta^\star)+
\alpha\gamma\Phi^\top DP\Pi^{\mu_k}\Phi\,(\bar\theta_k-\theta^\star).
\end{equation*}
The target row follows the updated online error:
\begin{align*}
\bar\theta_{k+1}-\theta^\star
&=(1-\beta)(\bar\theta_k-\theta^\star)+\beta(\theta_{k+1}-\theta^\star)\\
&=\beta(I-\alpha\Phi^\top D\Phi)(\theta_k-\theta^\star)+
\left((1-\beta)I+\beta\alpha\gamma\Phi^\top DP\Pi^{\mu_k}\Phi\right)
(\bar\theta_k-\theta^\star).
\end{align*}
Define the after-online-update stochastic mode by
\begin{equation*}
A_{\mu_k}^{\mathrm{SDLQL1}}
:=
\begin{bmatrix}
I-\alpha\Phi^\top D\Phi & \alpha\gamma\Phi^\top DP\Pi^{\mu_k}\Phi\\
\beta(I-\alpha\Phi^\top D\Phi) &
(1-\beta)I+\beta\alpha\gamma\Phi^\top DP\Pi^{\mu_k}\Phi
\end{bmatrix}.
\end{equation*}
With $\mu_k$ replaced by a generic stochastic policy $\mu$, the same formula
defines $A_\mu^{\mathrm{SDLQL1}}$. Hence this ordering has the exact SLS representation
\begin{align*}
\begin{bmatrix}
\theta_{k+1}-\theta^\star\\
\bar\theta_{k+1}-\theta^\star
\end{bmatrix}
=
A_{\mu_k}^{\mathrm{SDLQL1}}
\begin{bmatrix}
\theta_k-\theta^\star\\
\bar\theta_k-\theta^\star
\end{bmatrix}.
\end{align*}
For deterministic policies, let us define the corresponding finite switching family
\begin{equation*}
\mathcal A^{\mathrm{SDLQL1}}
:=
\set{
A_{\pi}^{\mathrm{SDLQL1}} :
\pi\in\Theta
}.
\end{equation*}
This after-online-update ordering is governed by the augmented switching family
$\mathcal A^{\mathrm{SDLQL1}}$. The standard certificate applies directly to the
augmented online-target error. Since
$A_{\mu_k}^{\mathrm{SDLQL1}}\in\co(\mathcal A^{\mathrm{SDLQL1}})$ by the state-wise convex-combination property of stochastic policies, the condition
$\rho(\mathcal A^{\mathrm{SDLQL1}})<1$ and~\Cref{lem:standard_jsr_convergence} imply
convergence to $(\theta^\star,\theta^\star)$ from every initialization. The same
lemma gives the finite-time $p_\epsilon$ and Euclidean bounds on the augmented
error with
\begin{align*}
x_k=
\begin{bmatrix}
\theta_k-\theta^\star\\
\bar\theta_k-\theta^\star
\end{bmatrix},
\qquad
\mathcal H=\mathcal A^{\mathrm{SDLQL1}}.
\end{align*}
Each block error is bounded by this augmented norm.

The endpoint behavior is useful before the small- and large-gain arguments are
stated. The choice $\beta=0$ freezes the target and prevents exponential
convergence of the full augmented error, while the choice $\beta=1$ copies the
updated online parameter into the target and connects this ordering to DLQL. 
\begin{lemma}
\label{lem:soft_after_endpoint_identities}
\label{prop:soft_after_beta_zero_endpoint}
\label{lem:soft_after_beta_one_endpoint}
For the after-online-update ordering, the following endpoint identities hold.
\begin{enumerate}[label=(\roman*)]
\item At $\beta=0$, we have
\begin{align*}
A_{\pi}^{\mathrm{SDLQL1}}=
\begin{bmatrix}
I-\alpha\Phi^\top D\Phi & \alpha\Phi^\top D\Phi A_\pi^{\mathrm{PQVI}}\\
0&I
\end{bmatrix},
\qquad \pi\in\Theta,
\end{align*}
and
\begin{align*}
\rho(\mathcal A^{\mathrm{SDLQL1}})
=
\max\{\rho(I-\alpha\Phi^\top D\Phi),1\}\ge 1.
\end{align*}
Consequently, an exactly frozen target cannot yield exponential convergence of
the full augmented online-target system.
\item At $\beta=1$,
\begin{align}
\rho(\mathcal A^{\mathrm{SDLQL1}})=\rho(\mathcal A^{\mathrm{DLQL}}).
\label{eq:beta_one_after_update_jsr}
\end{align}
\end{enumerate}
\end{lemma}
\begin{proof}
First set $\beta=0$. Substituting this value into the after-online-update
deterministic mode definition of $\mathcal A^{\mathrm{SDLQL1}}$ gives the displayed
block upper-triangular matrix. The diagonal blocks are the common matrix
$I-\alpha\Phi^\top D\Phi$ and the identity target block. Products of such modes
therefore have diagonal blocks $(I-\alpha\Phi^\top D\Phi)^k$ and $I$, so the JSR is
$\max\{\rho(I-\alpha\Phi^\top D\Phi),1\}$. This proves the frozen-target identity
and shows that the full augmented error cannot be uniformly exponentially stable
when the target is exactly frozen.
Now set $\beta=1$. Each deterministic augmented mode is
\begin{align*}
A_\pi^{\mathrm{SDLQL1}}
=\begin{bmatrix}
I-\alpha\Phi^\top D\Phi & \alpha\Phi^\top D\Phi A_\pi^{\mathrm{PQVI}}\\
I-\alpha\Phi^\top D\Phi & \alpha\Phi^\top D\Phi A_\pi^{\mathrm{PQVI}}
\end{bmatrix}.
\end{align*}
Use the invertible block transformation
\begin{align*}
T:=\begin{bmatrix} I&0\\ -I&I\end{bmatrix},
\qquad
T^{-1}=\begin{bmatrix} I&0\\ I&I\end{bmatrix}.
\end{align*}
A direct multiplication gives
\begin{align*}
T A_\pi^{\mathrm{SDLQL1}}T^{-1}
&=
\begin{bmatrix}
I-\alpha\Phi^\top D\Phi+\alpha\Phi^\top D\Phi A_\pi^{\mathrm{PQVI}} &
\alpha\Phi^\top D\Phi A_\pi^{\mathrm{PQVI}}\\
0&0
\end{bmatrix}\\
&=
\begin{bmatrix}
A_\pi^{\mathrm{DLQL}} & \alpha\Phi^\top D\Phi A_\pi^{\mathrm{PQVI}}\\
0&0
\end{bmatrix},
\end{align*}
where the second equality uses~\Cref{eq:A_mu_definition,eq:pqvi_mode_definition}.
The JSR is invariant under a common similarity transformation, and the JSR of a
block upper-triangular family is the maximum of the JSRs of the diagonal block
families~\cite[Property~6]{cicone2015note}. Hence
\begin{align*}
\rho(\mathcal A^{\mathrm{SDLQL1}})
&=
\rho\left(\set{T A_\pi^{\mathrm{SDLQL1}}T^{-1}:\pi\in\Theta}\right)\\
&=\max\left\{\rho(\mathcal A^{\mathrm{DLQL}}),\rho(\set{0})\right\}
=\rho(\mathcal A^{\mathrm{DLQL}}),
\end{align*}
which proves~\Cref{eq:beta_one_after_update_jsr}.
\end{proof}

For $0<\beta\ll1$, under the standing step-size
condition in~\Cref{ass:hard_target_stepsize}, the online parameter rapidly
approaches the frozen-target quasi-steady response, and hence, small positive $\beta$ is favorable precisely in the slow-target regime where the PQVI maps are JSR-stable. In contrast, $\beta$ close to one is favorable when
the DLQL endpoint is stable: in the after-online-update ordering, the endpoint
$\beta=1$ has the same JSR as the DLQL family by~\Cref{lem:soft_after_endpoint_identities}. Therefore if
DLQL is the unstable endpoint, taking $\beta$ close to one should not be expected
to help; if DLQL is stable, a large $\beta$ can inherit that favorable endpoint.
The next theorem makes the small-positive-$\beta$ statement precise for this
ordering by turning the PQVI endpoint stability into augmented-family stability.
\begin{theorem}
\label{thm:soft_after_small_beta_pqvi_stability}
Suppose~\Cref{ass:hard_target_stepsize} holds and the PQVI family is
JSR-stable, $\rho(\mathcal A^{\mathrm{PQVI}})<1$. Then there exists $\beta_0\in(0,1]$ such
that, for every $0<\beta\leq\beta_0$,
\begin{align*}
\rho(\mathcal A^{\mathrm{SDLQL1}})<1.
\end{align*}
Consequently, for every sufficiently small positive target gain, the
after-online-update SDLQL1 augmented error recursion is uniformly
exponentially stable under arbitrary switching. At $\beta=0$, however,
$\rho(\mathcal A^{\mathrm{SDLQL1}})=1$, so the conclusion is genuinely a
small-positive-$\beta$ statement.
\end{theorem}

\begin{proof}
See~\Cref{app:proof:thm:soft_after_small_beta_pqvi_stability}.
\end{proof}

The previous theorem addresses the slow-target regime near $\beta=0$. The
complementary high-gain regime is governed by the endpoint $\beta=1$. In the
after-online-update ordering, unlike the before-online-update ordering,
\Cref{eq:beta_one_after_update_jsr} identifies this endpoint exactly with DLQL.
Therefore the DLQL-stable case persists for all $\beta$ sufficiently close to
one. The theorem below states this endpoint-continuity argument for the full
after-update augmented family.
\begin{theorem}
\label{thm:soft_after_large_beta_dlql_stability}
Suppose the DLQL family is JSR-stable, $\rho(\mathcal A^{\mathrm{DLQL}})<1$.
Then there exists $\beta_1\in[0,1)$ such that, for every
$\beta_1\leq\beta\leq1$,
\begin{align*}
\rho(\mathcal A^{\mathrm{SDLQL1}})<1.
\end{align*}
Consequently, for all target gains sufficiently close to one, the
after-online-update SDLQL1 augmented error recursion is uniformly exponentially
stable under arbitrary switching.
\end{theorem}
\begin{proof}
At the endpoint $\beta=1$, \Cref{eq:beta_one_after_update_jsr} gives
\begin{align*}
\left.\rho(\mathcal A^{\mathrm{SDLQL1}})\right|_{\beta=1}
=\rho(\mathcal A^{\mathrm{DLQL}})<1.
\end{align*}
For the finite policy set $\Theta$, each entry of $A_\pi^{\mathrm{SDLQL1}}$ is an
affine function of $\beta$, and the JSR of a finite family is continuous with
respect to the matrix entries~\citep{heilstrang1995continuity,jungers2009joint}. Therefore there exists $\delta>0$ such that, for
$1-\delta\leq\beta\leq1$,
\begin{align*}
\left|\left.\rho(\mathcal A^{\mathrm{SDLQL1}})\right|_{\beta}
-\left.\rho(\mathcal A^{\mathrm{SDLQL1}})\right|_{\beta=1}\right|
<\frac{1-\rho(\mathcal A^{\mathrm{DLQL}})}{2}.
\end{align*}
Hence, for $1-\delta\leq\beta\leq1$,
\begin{align*}
\left.\rho(\mathcal A^{\mathrm{SDLQL1}})\right|_{\beta}
&<\rho(\mathcal A^{\mathrm{DLQL}})
+\frac{1-\rho(\mathcal A^{\mathrm{DLQL}})}{2}\\
&=\frac{1+\rho(\mathcal A^{\mathrm{DLQL}})}{2}<1.
\end{align*}
Taking $\beta_1:=\max\{0,1-\delta\}$ gives the claimed large-gain interval. The
uniform exponential stability statement follows from
\Cref{lem:standard_jsr_convergence} applied to the augmented error recursion.
\end{proof}

\section{SDLQL2: Soft Target Averaging Before the Online Update}

Having analyzed the after-online-update soft target ordering, we next consider
\emph{soft DLQL with target averaging before the online update}. Existing soft
update analyses and applications typically specify a target-tracking rule but do
not separate the spectral effect of this ordering from the after-online-update
ordering~\citep{zhang2021breaking,fellows2023why}. Throughout this
section, we use the step-size condition in~\Cref{ass:hard_target_stepsize}. In
this variant the target parameter tracks the current online parameter by the
soft averaging update
\begin{equation*}
\bar\theta_{k+1}=(1-\beta)\bar\theta_k+\beta\theta_k,
\qquad 0\leq \beta\leq 1,
\end{equation*}
where $\beta$ is an independent \emph{target update gain}. Therefore, the online Q-learning
step-size is $\alpha$, while the target tracking gain is $\beta$.
This ordering differs from the after-online-update version in the preceding
section: here the target is averaged before it is used in the next online
analysis step. The SDLQL2 version separates the online parameter $\theta_k$
from the target parameter $\bar\theta_k$:
\begin{align*}
\theta_{k+1}
&=
\theta_k+
\alpha\Phi^\top D\left(R+\gamma P V_{\bar\theta_k}-\Phi\theta_k\right),\\
\bar\theta_{k+1}
&=
(1-\beta)\bar\theta_k+\beta\theta_k.
\end{align*}
The online update uses the current target parameter, while the target itself is updated by a separate soft averaging step.
The augmented map is
\begin{align}
g(\theta,\bar\theta)
:=
\begin{bmatrix}
\theta+
\alpha\Phi^\top D(R+\gamma P V_{\bar\theta}-\Phi\theta)\\
(1-\beta)\bar\theta+\beta\theta
\end{bmatrix}.
\label{eq:aug_map}
\end{align}
Because $\theta^\star$ is the projected Bellman fixed point,
$(\theta^\star,\theta^\star)$ is an augmented fixed point of~\Cref{eq:aug_map}.
By~\Cref{eq:stoch_policy_linearization}, for every $k$ there exists a stochastic
policy $\mu_k=\mu_{\bar\theta_k,\theta^\star}$ such that $V_{\bar\theta_k}-V_{\theta^\star}
=
\Pi^{\mu_k}\Phi(\bar\theta_k-\theta^\star)$. Subtracting the fixed-point equation gives
\begin{align*}
\theta_{k+1}-\theta^\star
&=(I-\alpha\Phi^\top D\Phi)(\theta_k-\theta^\star)+
\alpha\gamma\Phi^\top DP\Pi^{\mu_k}\Phi\,(\bar\theta_k-\theta^\star),\\
\bar\theta_{k+1}-\theta^\star
&=\beta(\theta_k-\theta^\star)+(1-\beta)(\bar\theta_k-\theta^\star).
\end{align*}
Therefore, the SDLQL2 recursion has the exact switched linear representation
\begin{align*}
\begin{bmatrix}
\theta_{k+1}-\theta^\star\\
\bar\theta_{k+1}-\theta^\star
\end{bmatrix}
=A_{\mu_k}^{\mathrm{SDLQL2}}
\begin{bmatrix}
\theta_k-\theta^\star\\
\bar\theta_k-\theta^\star
\end{bmatrix},
\end{align*}
where
\begin{equation*}
A_{\mu}^{\mathrm{SDLQL2}}
:=
\begin{bmatrix}
I-\alpha\Phi^\top D\Phi & \alpha\gamma\Phi^\top DP\Pi^\mu\Phi\\
\beta I & (1-\beta)I
\end{bmatrix}.
\end{equation*}
For deterministic policies, define the finite family
\begin{align}
\mathcal A^{\mathrm{SDLQL2}}
:=\set{
A_{\pi}^{\mathrm{SDLQL2}}: \pi\in\Theta
},
\label{eq:soft_finite_family}
\end{align}
and its JSR is $\rho(\mathcal A^{\mathrm{SDLQL2}})$.
The preceding switched representation gives the SDLQL2 convergence certificate
through the standard lemma. Since every stochastic mode generated by the Bellman
maximum lies in $\co(\mathcal A^{\mathrm{SDLQL2}})$ by the state-wise convex-combination property of stochastic policies, the condition
$\rho(\mathcal A^{\mathrm{SDLQL2}})<1$ and~\Cref{lem:standard_jsr_convergence} imply
convergence to $(\theta^\star,\theta^\star)$ from every initialization. The same
lemma gives the finite-time $p_\epsilon$ and Euclidean bounds on the augmented
error with
\begin{align*}
x_k=
\begin{bmatrix}
\theta_k-\theta^\star\\
\bar\theta_k-\theta^\star
\end{bmatrix},
\qquad
\mathcal H=\mathcal A^{\mathrm{SDLQL2}}.
\end{align*}
Each block error is again bounded by the augmented norm.
Before using the small-positive target gain as a stability mechanism, it is
helpful to separate the degenerate frozen-target endpoint. The next lemma
states the unit spectral obstruction at $\beta=0$, now for the before-update
ordering.
\begin{lemma}
\label{lem:soft_before_beta_zero_endpoint}
At $\beta=0$, the before-online-update SDLQL2 family satisfies
\begin{align*}
A_{\pi}^{\mathrm{SDLQL2}}=
\begin{bmatrix}
I-\alpha\Phi^\top D\Phi & \alpha\Phi^\top D\Phi A_\pi^{\mathrm{PQVI}}\\
0&I
\end{bmatrix},
\qquad \pi\in\Theta,
\end{align*}
and $\rho(\mathcal A^{\mathrm{SDLQL2}})=\max\{\rho(I-\alpha\Phi^\top D\Phi),1\} \ge 1$.
\end{lemma}
\begin{proof}
Substituting $\beta=0$ into the deterministic mode definition
~\Cref{eq:soft_finite_family} gives
\begin{align*}
A_{\pi}^{\mathrm{SDLQL2}}
=
\begin{bmatrix}
I-\alpha\Phi^\top D\Phi & \alpha\Phi^\top D\Phi A_\pi^{\mathrm{PQVI}}\\
0&I
\end{bmatrix},
\qquad \pi\in\Theta.
\end{align*}
By the block upper-triangular JSR property~\cite[Property~6]{cicone2015note},
\begin{align*}
\rho(\mathcal A^{\mathrm{SDLQL2}})
&=\max\left\{
\rho\left(\set{I-\alpha\Phi^\top D\Phi:\pi\in\Theta}\right),
\rho\left(\set{I:\pi\in\Theta}\right)
\right\}\\
&=\max\{\rho(I-\alpha\Phi^\top D\Phi),1\}.
\end{align*}
This quantity is at least one, so no uniform exponential contraction of the full
augmented state is possible at $\beta=0$.
\end{proof}
Consequently, an exactly frozen target cannot yield exponential convergence of
the full augmented online-target system. The meaningful target-averaging regime is therefore $0<\beta\ll1$.
Under the standing step-size condition in~\Cref{ass:hard_target_stepsize}, the
fixed-target online update is Schur-stable. Then, under the time-scale separation
induced by a small target gain, the online error rapidly approaches the
quasi-steady state as in the previous section.
Therefore, the relevant small-$\beta$ stability object is close to the PQVI family
$\mathcal A^{\mathrm{PQVI}}$. The following theorem makes this small-positive-$\beta$ statement precise in the
case where the reduced PQVI family is JSR-stable, matching the role of the
analogous after-update result.
\begin{theorem}
\label{thm:soft_small_beta_pqvi_stability}
Suppose~\Cref{ass:hard_target_stepsize} holds and the PQVI family is
JSR-stable, $\rho(\mathcal A^{\mathrm{PQVI}})<1$. Then there exists $\beta_0\in(0,1]$ such
that, for every $0<\beta\leq\beta_0$,
\begin{align*}
\rho(\mathcal A^{\mathrm{SDLQL2}})<1.
\end{align*}
Consequently, for every sufficiently small positive target gain, the SDLQL2 augmented error recursion is uniformly exponentially stable under arbitrary switching. At $\beta=0$, however,
$\rho(\mathcal A^{\mathrm{SDLQL2}})=1$, so the conclusion is genuinely a
small-positive-$\beta$ statement.
\end{theorem}
\begin{proof}
See~\Cref{app:proof:thm:soft_small_beta_pqvi_stability}.
\end{proof}
The preceding theorem covers the slow-target side, where $\beta$ is close to
zero and the reduced PQVI family governs the target motion. The opposite
large-gain side must be read with more care in the before-online-update
ordering. When $\beta=1$, \Cref{eq:soft_finite_family} does not collapse to
the DLQL family; it keeps a one-step delayed target coordinate. Therefore the
DLQL stability condition alone is not a large-gain stability certificate for
this ordering. In particular, even if $\rho(\mathcal A^{\mathrm{DLQL}})<1$, taking $\beta$ sufficiently close to one does not by itself guarantee $\rho(\mathcal A^{\mathrm{SDLQL2}})<1$.
A large-gain conclusion would require a separate stability certificate for the
delayed $\beta=1$ augmented SLS, rather than only the SLS of the DLQL endpoint.

Such SDLQL2 examples are easy to find in scalar settings: the augmented SDLQL2 JSR can be strictly below the DLQL JSR, including cases with
$\rho(\mathcal A^{\mathrm{SDLQL2}})<1<\rho(\mathcal A^{\mathrm{DLQL}})$. This supports the
interpretation that SDLQL2 changes the feedback path rather than
merely reducing the DLQL step-size.

\section{Conclusion}

This paper analyzed periodic hard target updates and soft target averaging in
linear Q-learning through their exact switched linear error dynamics. The main
message is that target updates change the feedback path of the Bellman maximum,
so their stabilizing effect is captured by the JSR of the induced target-update
family rather than by a simple rescaling of the online step-size.

For hard targets, the period-$m$ boundary map shows that $m$-DLQL interpolates
between DLQL at $m=1$ and PQVI as $m\to\infty$ under the frozen-target step-size
condition. Therefore, PQVI stability certifies all sufficiently large target
periods, whereas DLQL stability certifies the period-one endpoint; intermediate
periods require their own JSR certificate. For soft targets, both update
orderings lead to augmented online-target systems. In the small-positive-gain
regime, PQVI stability gives augmented stability for both SDLQL1 and SDLQL2,
while the after-online-update ordering also inherits DLQL stability for target
gains sufficiently close to one.

The results are conditional but exact: when the stated spectral and step-size
conditions hold, the deterministic error converges to zero and the iterates reach
the projected Q-Bellman solution rather than a residual neighborhood. The same
structure identifies the mean dynamics that must be controlled when unbiased
stochastic sampling errors are added. This provides a precise switched-system
explanation of why target updates can stabilize linear Q-learning and clarifies
which endpoint certificate is responsible for each certified target-update
regime.


\appendix
\begin{center}
{\Huge\bfseries Appendix}
\end{center}
\addcontentsline{toc}{section}{Appendix}

\section{Tabular Case: PQVI, DLQL, and \texorpdfstring{$m$}{m}-DLQL}
\label{app:tabular_pqvi_vs_dlql}
This appendix compares the tabular specializations of PQVI, one-period DLQL, and
$m$-DLQL. In the usual synchronous tabular setting considered in this appendix, the
feature representation is the identity and the projection is the identity. Thus
the PQVI mode associated with a deterministic policy $\pi$ is
$A_\pi^{\mathrm{PQVI}}=\gamma P\Pi^\pi$. For a
switching signal $\pi_k$, the first two update rules are
\begin{align*}
\text{PQVI:}\qquad
\theta_{k+1}&=R+\gamma P\Pi^{\pi_k}\theta_k,\\
\text{DLQL:}\qquad
\theta_{k+1}&=\theta_k+\alpha\left(R+\gamma P\Pi^{\pi_k}\theta_k-\theta_k\right).
\end{align*}
After subtracting the common tabular Bellman fixed point, these become the
switching systems
\begin{align*}
\text{PQVI:}\qquad
x_{k+1}&=A_{\pi_k}^{\mathrm{PQVI}}x_k,
&
A_\pi^{\mathrm{PQVI}}&=\gamma P\Pi^\pi,\\
\text{DLQL:}\qquad
x_{k+1}&=A_{\pi_k}^{\mathrm{DLQL}}x_k,
&
A_\pi^{\mathrm{DLQL}}&=(1-\alpha)I+\alpha A_\pi^{\mathrm{PQVI}}.
\end{align*}
The period-$m$ hard-target map is obtained by freezing the target policy for
$m$ online steps and then reading the error at the next target boundary. In the
tabular case this gives
\begin{align*}
A_{\pi,m}^{\mathrm{DLQL}}
&=(1-\alpha)^m I+
\sum_{i=0}^{m-1}(1-\alpha)^i\alpha A_\pi^{\mathrm{PQVI}}.
\end{align*}
Therefore, these switching families are the synchronous tabular counterparts of
$\mathcal A^{\mathrm{PQVI}}$, $\mathcal A^{\mathrm{DLQL}}$, and
$\mathcal A_m^{\mathrm{DLQL}}$ from the main text. The proposition below
summarizes the row-stochastic comparison for all finite hard-target periods. In
this tabular specialization, the nonnegativity and common row-sum identities
follow from the fact that $P\Pi^\pi$ is row-stochastic; they are not additional
assumptions.
\begin{proposition}
\label{prop:tabular_hard_target_comparison}
In the tabular setting above, let $0<\gamma<1$ and $0<\alpha\leq1$. Then, for every integer
$m\geq1$,
\begin{align*}
\rho(\mathcal A^{\mathrm{PQVI}})&=\gamma,\\
\rho(\mathcal A^{\mathrm{DLQL}})&=1-\alpha(1-\gamma),\\
\rho(\mathcal A_m^{\mathrm{DLQL}})&=\gamma+(1-\gamma)(1-\alpha)^m.
\end{align*}
Consequently, PQVI and DLQL have the same tabular JSR when $\alpha=1$. For
$0<\alpha<1$, the period-$m$ hard-target JSR decreases from the one-period DLQL
value to the PQVI value as $m\to\infty$.
\end{proposition}
\begin{proof}
We use the infinity norm in the JSR definition from~\Cref{def:jsr}. First note a
simple common row-sum calculation. If a finite family $\mathcal H_c$ consists of
nonnegative matrices whose row sums are all equal to the same scalar $c\geq0$,
then every length-$k$ product of matrices from $\mathcal H_c$ is nonnegative and
has row sums $c^k$. Hence every such product has infinity norm exactly $c^k$.
Therefore
\begin{align*}
\rho(\mathcal H_c)
&=\lim_{k\to\infty}
\max_{B_1,\ldots,B_k\in\mathcal H_c}
\left\|B_k\cdots B_1\right\|_\infty^{1/k}\\
&=\lim_{k\to\infty} (c^k)^{1/k}=c.
\end{align*}
This identity is now applied to the three tabular families.
For PQVI, $P\Pi^\pi$ is row-stochastic for every deterministic policy $\pi$.
Therefore, $A_\pi^{\mathrm{PQVI}}=\gamma P\Pi^\pi$ is nonnegative with common row sum
$\gamma$, and the preceding JSR calculation gives
$\rho(\mathcal A^{\mathrm{PQVI}})=\gamma$.
For one-period DLQL,
\begin{align*}
A_\pi^{\mathrm{DLQL}}=(1-\alpha)I+\alpha A_\pi^{\mathrm{PQVI}}.
\end{align*}
Because $0<\alpha\leq1$, these matrices are nonnegative. Their common row sum is
\begin{align*}
(1-\alpha)+\alpha\gamma=1-\alpha(1-\gamma).
\end{align*}
The same infinity-norm JSR calculation gives
\begin{align*}
\rho(\mathcal A^{\mathrm{DLQL}})=1-\alpha(1-\gamma).
\end{align*}
For the period-$m$ hard-target family, the tabular boundary map is
\begin{align*}
A_{\pi,m}^{\mathrm{DLQL}}
&=(1-\alpha)^m I+
\sum_{i=0}^{m-1}(1-\alpha)^i\alpha A_\pi^{\mathrm{PQVI}}.
\end{align*}
Again $0<\alpha\leq1$ implies nonnegativity. The common row sum is
\begin{align*}
(1-\alpha)^m+
\sum_{i=0}^{m-1}(1-\alpha)^i\alpha\gamma
&=(1-\alpha)^m+
\gamma\left[1-(1-\alpha)^m\right]\\
&=\gamma+(1-\gamma)(1-\alpha)^m.
\end{align*}
Using the JSR definition with the infinity norm once more yields
\begin{align*}
\rho(\mathcal A_m^{\mathrm{DLQL}})=\gamma+(1-\gamma)(1-\alpha)^m.
\end{align*}
The endpoint and monotonicity statements follow immediately: when $\alpha=1$ the
DLQL row sum equals $\gamma$, and when $0<\alpha<1$, the term $(1-\alpha)^m$
decreases to zero as $m\to\infty$.
\end{proof}

\section{Argument for~\Cref{thm:soft_after_small_beta_pqvi_stability}}
\label{app:proof:thm:soft_after_small_beta_pqvi_stability}

The argument uses two norms, one for the PQVI family and one for the fixed matrix
$I-\alpha\Phi^\top D\Phi$. Before combining their estimates, we state the standard
finite-dimensional norm-equivalence fact that allows the constants in the cross
terms to be chosen uniformly over a finite family of modes; see, for example,
\citet[Theorem~2.4-5]{kreyszig1978introductory}.
\begin{lemma}
\label{lem:appendix_norm_equivalence}
Let $p$ and $q$ be norms on a finite-dimensional real vector space. Then there
exist constants $c_1,c_2>0$ such that
\begin{align*}
c_1p(u)\leq q(u)\leq c_2p(u)
\qquad\text{for every }u.
\end{align*}
Consequently, every linear map between finite-dimensional normed spaces is
bounded with respect to any chosen pair of norms, and a finite family of such
maps admits a common bound.
\end{lemma}
\begin{proof}
The unit sphere of one norm is compact in finite dimension, and the other norm is
continuous on that sphere. Its positive minimum and finite maximum give the two
comparison constants. The boundedness of a linear map follows by applying this
comparison to the Euclidean operator norm, and the maximum of finitely many such
bounds gives a common bound for a finite family.
\end{proof}

Write the augmented error variables as
\begin{align*}
x_k:=\theta_k-\theta^\star,
\qquad
y_k:=\bar\theta_k-\theta^\star.
\end{align*}
For a deterministic policy $\pi$ with PQVI mode
$A_\pi^{\mathrm{PQVI}}\in\mathcal A^{\mathrm{PQVI}}$, the after-online-update
SDLQL1 recursion is
\begin{align*}
x_{k+1}&=(I-\alpha\Phi^\top D\Phi)x_k+
\alpha\Phi^\top D\Phi A_\pi^{\mathrm{PQVI}} y_k,\\
y_{k+1}&=(1-\beta)y_k+\beta x_{k+1}.
\end{align*}
At $\beta=0$, the target error is frozen, so the augmented family has an
identity target block and
\begin{align*}
\rho(\mathcal A^{\mathrm{SDLQL1}})
=\max\{\rho(I-\alpha\Phi^\top D\Phi),1\}=1
\end{align*}
by~\Cref{ass:hard_target_stepsize}. It remains to prove stability for
sufficiently small positive $\beta$.

The matrices $I-\alpha\Phi^\top D\Phi$ and $\alpha\Phi^\top D\Phi$ commute, and
$\alpha\Phi^\top D\Phi$ is nonsingular. Introduce the near-identity slow
coordinate
\begin{align*}
z_k:=y_k+\beta(I-\alpha\Phi^\top D\Phi)(\alpha\Phi^\top D\Phi)^{-1}x_k.
\end{align*}
The coordinate change $(x_k,y_k)\mapsto(x_k,z_k)$ is invertible for each fixed
$\beta$, with inverse
\begin{align*}
y_k=z_k-\beta(I-\alpha\Phi^\top D\Phi)(\alpha\Phi^\top D\Phi)^{-1}x_k.
\end{align*}
The two elementary identities used below are
\begin{align*}
\left[I+(I-\alpha\Phi^\top D\Phi)(\alpha\Phi^\top D\Phi)^{-1}\right]
\alpha\Phi^\top D\Phi
&=\alpha\Phi^\top D\Phi
+(I-\alpha\Phi^\top D\Phi)\\
&=I,
\end{align*}
and
\begin{align*}
(I-\alpha\Phi^\top D\Phi)(\alpha\Phi^\top D\Phi)^{-1}
\alpha\Phi^\top D\Phi=I-\alpha\Phi^\top D\Phi.
\end{align*}
Also,
\begin{align*}
I+(I-\alpha\Phi^\top D\Phi)(\alpha\Phi^\top D\Phi)^{-1}
&=(\alpha\Phi^\top D\Phi)(\alpha\Phi^\top D\Phi)^{-1}
+(I-\alpha\Phi^\top D\Phi)(\alpha\Phi^\top D\Phi)^{-1}\\
&=(\alpha\Phi^\top D\Phi)^{-1},
\end{align*}
where the last equality factors the common right multiplier. Hence, using the
commutativity of $I-\alpha\Phi^\top D\Phi$ and $\alpha\Phi^\top D\Phi$,
\begin{align*}
\left[I+(I-\alpha\Phi^\top D\Phi)(\alpha\Phi^\top D\Phi)^{-1}\right]
(I-\alpha\Phi^\top D\Phi)
=(I-\alpha\Phi^\top D\Phi)(\alpha\Phi^\top D\Phi)^{-1}.
\end{align*}

Substituting the inverse formula for $y_k$ into the online row gives
\begin{align*}
x_{k+1}
&=(I-\alpha\Phi^\top D\Phi)x_k+
\alpha\Phi^\top D\Phi A_\pi^{\mathrm{PQVI}}
\left[z_k-\beta(I-\alpha\Phi^\top D\Phi)(\alpha\Phi^\top D\Phi)^{-1}x_k\right]\\
&=\left(I-\alpha\Phi^\top D\Phi
-\beta\alpha\Phi^\top D\Phi A_\pi^{\mathrm{PQVI}}
(I-\alpha\Phi^\top D\Phi)(\alpha\Phi^\top D\Phi)^{-1}\right)x_k\\
&\qquad +\alpha\Phi^\top D\Phi A_\pi^{\mathrm{PQVI}} z_k.
\end{align*}
For the slow coordinate, start from its definition at time $k+1$ and use the
after-online-update target row:
\begin{align*}
z_{k+1}
&=y_{k+1}+\beta(I-\alpha\Phi^\top D\Phi)(\alpha\Phi^\top D\Phi)^{-1}x_{k+1}\\
&=(1-\beta)y_k+
\beta\left[I+(I-\alpha\Phi^\top D\Phi)(\alpha\Phi^\top D\Phi)^{-1}\right]x_{k+1}\\
&=(1-\beta)y_k+
\beta\left[I+(I-\alpha\Phi^\top D\Phi)(\alpha\Phi^\top D\Phi)^{-1}\right]
\left[(I-\alpha\Phi^\top D\Phi)x_k+
\alpha\Phi^\top D\Phi A_\pi^{\mathrm{PQVI}}y_k\right]\\
&=(1-\beta)y_k
+\beta(I-\alpha\Phi^\top D\Phi)(\alpha\Phi^\top D\Phi)^{-1}x_k
+\beta A_\pi^{\mathrm{PQVI}}y_k\\
&=y_k+\beta(I-\alpha\Phi^\top D\Phi)(\alpha\Phi^\top D\Phi)^{-1}x_k
+\beta(A_\pi^{\mathrm{PQVI}}-I)y_k\\
&=z_k+\beta(A_\pi^{\mathrm{PQVI}}-I)y_k\\
&=z_k+\beta(A_\pi^{\mathrm{PQVI}}-I)
\left[z_k-\beta(I-\alpha\Phi^\top D\Phi)(\alpha\Phi^\top D\Phi)^{-1}x_k\right]\\
&=\left[I+\beta(A_\pi^{\mathrm{PQVI}}-I)\right]z_k
-\beta^2(A_\pi^{\mathrm{PQVI}}-I)(I-\alpha\Phi^\top D\Phi)(\alpha\Phi^\top D\Phi)^{-1}x_k.
\end{align*}
Therefore, the transformed dynamics have the displayed fast row and slow row exactly.

Because $\rho(\mathcal A^{\mathrm{PQVI}})<1$, there are a norm $p$ and a number
$\eta\in(0,1)$ such that
\begin{align*}
p(A_\pi^{\mathrm{PQVI}}u)\leq \eta p(u),
\qquad \forall A_\pi^{\mathrm{PQVI}}\in\mathcal A^{\mathrm{PQVI}},
\qquad \forall u.
\end{align*}
Let $c:=1-\eta>0$. Then, for every $0\leq\beta\leq1$,
\begin{align*}
p\left(\left[I+\beta(A_\pi^{\mathrm{PQVI}}-I)\right]u\right)
&=p\left((1-\beta)u+\beta A_\pi^{\mathrm{PQVI}}u\right)\\
&\leq (1-\beta)p(u)+\beta p(A_\pi^{\mathrm{PQVI}}u)\\
&\leq (1-\beta)p(u)+\beta\eta p(u)\\
&=\left[1-\beta(1-\eta)\right]p(u)\\
&=(1-c\beta)p(u).
\end{align*}
By~\Cref{ass:hard_target_stepsize}, there are a norm $q$ and a number
$\eta_q\in(0,1)$ such that
\begin{align*}
q\left((I-\alpha\Phi^\top D\Phi)u\right)\leq \eta_q q(u),
\qquad \forall u.
\end{align*}
By~\Cref{lem:appendix_norm_equivalence}, and since $\mathcal A^{\mathrm{PQVI}}$ is finite, there exist constants $a,d\geq0$ and $b>0$ such that, for every
$A_\pi^{\mathrm{PQVI}}\in\mathcal A^{\mathrm{PQVI}}$,
\begin{align*}
q\left(\alpha\Phi^\top D\Phi A_\pi^{\mathrm{PQVI}}(I-\alpha\Phi^\top D\Phi)(\alpha\Phi^\top D\Phi)^{-1}u\right)&\leq a q(u),\\
q\left(\alpha\Phi^\top D\Phi A_\pi^{\mathrm{PQVI}}u\right)&\leq b p(u),\\
p\left((A_\pi^{\mathrm{PQVI}}-I)(I-\alpha\Phi^\top D\Phi)(\alpha\Phi^\top D\Phi)^{-1}u\right)&\leq d q(u).
\end{align*}
Choose $\eta_{q,0}\in(\eta_q,1)$ and then restrict $\beta$ so that $\eta_q+a\beta\leq \eta_{q,0}$.
The transformed dynamics therefore satisfy
\begin{align*}
q(x_{k+1})&\leq \eta_{q,0} q(x_k)+b p(z_k),\\
p(z_{k+1})&\leq (1-c\beta)p(z_k)+d\beta^2 q(x_k).
\end{align*}
For $0<\beta\leq1$, define the weighted norm
\begin{align*}
W_\beta(x,z):=q(x)+\frac{2b}{c\beta}p(z).
\end{align*}
Then
\begin{align*}
W_\beta(x_{k+1},z_{k+1})
&\leq \eta_{q,0} q(x_k)+b p(z_k)
+\frac{2b}{c\beta}\left[(1-c\beta)p(z_k)+d\beta^2 q(x_k)\right]\\
&=\left(\eta_{q,0}+\frac{2bd}{c}\beta\right)q(x_k)
+\frac{2b}{c\beta}\left(1-\frac{c\beta}{2}\right)p(z_k).
\end{align*}
For all sufficiently small $\beta>0$,
\begin{align*}
\tau_\beta
:=
\max\left\{
 \eta_{q,0}+\frac{2bd}{c}\beta,
 1-\frac{c\beta}{2}
\right\}
<1.
\end{align*}
Therefore every transformed deterministic after-online-update mode satisfies the
explicit contraction estimate
\begin{align*}
W_\beta(x_{k+1},z_{k+1})
&\leq \tau_\beta q(x_k)+\tau_\beta\frac{2b}{c\beta}p(z_k)\\
&=\tau_\beta W_\beta(x_k,z_k).
\end{align*}
Since the coordinate change is invertible, the same contraction certificate gives
\begin{align*}
\rho(\mathcal A^{\mathrm{SDLQL1}})\leq \tau_\beta<1
\end{align*}
for all $0<\beta\leq\beta_0$ after decreasing $\beta_0$ if necessary. The same
argument also covers stochastic-policy modes because they lie in the convex hull
of the deterministic modes and the norm inequalities above are preserved under
convex combinations.

\section{Argument for~\Cref{thm:soft_small_beta_pqvi_stability}}
\label{app:proof:thm:soft_small_beta_pqvi_stability}

Write the augmented error variables as
\begin{align*}
x_k:=\theta_k-\theta^\star,
\qquad
y_k:=\bar\theta_k-\theta^\star.
\end{align*}
For a deterministic policy $\pi$ with PQVI mode
$A_\pi^{\mathrm{PQVI}}\in\mathcal A^{\mathrm{PQVI}}$, the before-online-update
SDLQL2 recursion is
\begin{align*}
x_{k+1}&=(I-\alpha\Phi^\top D\Phi)x_k+
\alpha\Phi^\top D\Phi A_\pi^{\mathrm{PQVI}} y_k,\\
y_{k+1}&=(1-\beta)y_k+\beta x_k.
\end{align*}
At $\beta=0$, each augmented mode is block upper triangular with diagonal blocks
$I-\alpha\Phi^\top D\Phi$ and $I$. By~\Cref{ass:hard_target_stepsize},
$\rho(I-\alpha\Phi^\top D\Phi)<1$, and hence
$\rho(\mathcal A^{\mathrm{SDLQL2}})=1$.

For the small-positive-$\beta$ part, introduce the near-identity slow coordinate
\begin{align*}
z_k:=y_k+\beta(\alpha\Phi^\top D\Phi)^{-1}x_k.
\end{align*}
The coordinate change $(x_k,y_k)\mapsto(x_k,z_k)$ is invertible for each fixed
$\beta$, with inverse
\begin{align*}
y_k=z_k-\beta(\alpha\Phi^\top D\Phi)^{-1}x_k.
\end{align*}
The online row follows by direct substitution:
\begin{align*}
x_{k+1}
&=(I-\alpha\Phi^\top D\Phi)x_k+
\alpha\Phi^\top D\Phi A_\pi^{\mathrm{PQVI}}
\left[z_k-\beta(\alpha\Phi^\top D\Phi)^{-1}x_k\right]\\
&=(I-\alpha\Phi^\top D\Phi)x_k+
\alpha\Phi^\top D\Phi A_\pi^{\mathrm{PQVI}}z_k
-\beta\alpha\Phi^\top D\Phi A_\pi^{\mathrm{PQVI}}(\alpha\Phi^\top D\Phi)^{-1}x_k.
\end{align*}
For the slow coordinate, start from its definition at time $k+1$:
\begin{align*}
z_{k+1}
&=y_{k+1}+\beta(\alpha\Phi^\top D\Phi)^{-1}x_{k+1}\\
&=(1-\beta)y_k+\beta x_k
+\beta(\alpha\Phi^\top D\Phi)^{-1}
\left[(I-\alpha\Phi^\top D\Phi)x_k+
\alpha\Phi^\top D\Phi A_\pi^{\mathrm{PQVI}}y_k\right]\\
&=(1-\beta)y_k+\beta x_k
+\beta\left[(\alpha\Phi^\top D\Phi)^{-1}-I\right]x_k
+\beta A_\pi^{\mathrm{PQVI}}y_k\\
&=y_k+\beta(\alpha\Phi^\top D\Phi)^{-1}x_k
+\beta(A_\pi^{\mathrm{PQVI}}-I)y_k\\
&=z_k+\beta(A_\pi^{\mathrm{PQVI}}-I)y_k\\
&=z_k+\beta(A_\pi^{\mathrm{PQVI}}-I)
\left[z_k-\beta(\alpha\Phi^\top D\Phi)^{-1}x_k\right]\\
&=\left[I+\beta(A_\pi^{\mathrm{PQVI}}-I)\right]z_k
-\beta^2(A_\pi^{\mathrm{PQVI}}-I)(\alpha\Phi^\top D\Phi)^{-1}x_k.
\end{align*}
The equality
$(\alpha\Phi^\top D\Phi)^{-1}(I-\alpha\Phi^\top D\Phi)
=(\alpha\Phi^\top D\Phi)^{-1}-I$ is the only algebraic cancellation used in the
third line.

Because $\rho(\mathcal A^{\mathrm{PQVI}})<1$, there are a norm $p$ and a number
$\eta\in(0,1)$ such that
\begin{align*}
p(A_\pi^{\mathrm{PQVI}}u)\leq \eta p(u),
\qquad \forall A_\pi^{\mathrm{PQVI}}\in\mathcal A^{\mathrm{PQVI}},
\qquad \forall u.
\end{align*}
Let $c:=1-\eta>0$. Then, for every $0\leq\beta\leq1$,
\begin{align*}
p\left(\left[I+\beta(A_\pi^{\mathrm{PQVI}}-I)\right]u\right)
&=p\left((1-\beta)u+\beta A_\pi^{\mathrm{PQVI}}u\right)\\
&\leq (1-\beta)p(u)+\beta p(A_\pi^{\mathrm{PQVI}}u)\\
&\leq (1-\beta)p(u)+\beta\eta p(u)\\
&=\left[1-\beta(1-\eta)\right]p(u)\\
&=(1-c\beta)p(u).
\end{align*}
By~\Cref{ass:hard_target_stepsize}, there are a norm $q$ and a number
$\eta_q\in(0,1)$ such that
\begin{align*}
q\left((I-\alpha\Phi^\top D\Phi)u\right)\leq \eta_q q(u),
\qquad \forall u.
\end{align*}
By~\Cref{lem:appendix_norm_equivalence}, and since $\mathcal A^{\mathrm{PQVI}}$ is finite, there exist constants $a,d\geq0$ and $b>0$ such that, for every
$A_\pi^{\mathrm{PQVI}}\in\mathcal A^{\mathrm{PQVI}}$,
\begin{align*}
q\left(\alpha\Phi^\top D\Phi A_\pi^{\mathrm{PQVI}}(\alpha\Phi^\top D\Phi)^{-1}u\right)&\leq a q(u),\\
q\left(\alpha\Phi^\top D\Phi A_\pi^{\mathrm{PQVI}}u\right)&\leq b p(u),\\
p\left((A_\pi^{\mathrm{PQVI}}-I)(\alpha\Phi^\top D\Phi)^{-1}u\right)&\leq d q(u).
\end{align*}
Choose $\eta_{q,0}\in(\eta_q,1)$ and then restrict $\beta$ so that $\eta_q+a\beta\leq \eta_{q,0}$.
The transformed dynamics therefore satisfy
\begin{align*}
q(x_{k+1})&\leq \eta_{q,0} q(x_k)+b p(z_k),\\
p(z_{k+1})&\leq (1-c\beta)p(z_k)+d\beta^2 q(x_k).
\end{align*}
For $0<\beta\leq1$, define the weighted norm
\begin{align*}
W_\beta(x,z):=q(x)+\frac{2b}{c\beta}p(z).
\end{align*}
Then
\begin{align*}
W_\beta(x_{k+1},z_{k+1})
&\leq \eta_{q,0} q(x_k)+b p(z_k)
+\frac{2b}{c\beta}\left[(1-c\beta)p(z_k)+d\beta^2 q(x_k)\right]\\
&=\left(\eta_{q,0}+\frac{2bd}{c}\beta\right)q(x_k)
+\frac{2b}{c\beta}\left(1-\frac{c\beta}{2}\right)p(z_k).
\end{align*}
For all sufficiently small $\beta>0$,
\begin{align*}
\tau_\beta
:=
\max\left\{
 \eta_{q,0}+\frac{2bd}{c}\beta,
 1-\frac{c\beta}{2}
\right\}
<1.
\end{align*}
Therefore every transformed deterministic before-online-update mode satisfies the
explicit contraction estimate
\begin{align*}
W_\beta(x_{k+1},z_{k+1})
&\leq \tau_\beta q(x_k)+\tau_\beta\frac{2b}{c\beta}p(z_k)\\
&=\tau_\beta W_\beta(x_k,z_k).
\end{align*}
Since the coordinate change is invertible, the same contraction certificate gives
\begin{align*}
\rho(\mathcal A^{\mathrm{SDLQL2}})\leq \tau_\beta<1
\end{align*}
for all $0<\beta\leq\beta_0$ after decreasing $\beta_0$ if necessary. The same
argument also covers stochastic-policy modes because they lie in the convex hull
of the deterministic modes and the norm inequalities above are preserved under
convex combinations.


\begin{thebibliography}{99}

\bibitem[Bertsekas and Tsitsiklis(1996)]{bertsekas1996neuro}
Dimitri~P. Bertsekas and John~N. Tsitsiklis.
\newblock \emph{Neuro-Dynamic Programming}.
\newblock Athena Scientific, Belmont, MA, 1996.

\bibitem[Blondel and Nesterov(2005)]{blondel2005computational}
Vincent~D. Blondel and Yurii Nesterov.
\newblock Computationally efficient approximations of the joint spectral radius.
\newblock \emph{SIAM Journal on Matrix Analysis and Applications},
27(1):256--272, 2005.

\bibitem[Borkar and Meyn(2000)]{borkarmeyn2000ode}
Vivek~S. Borkar and Sean~P. Meyn.
\newblock The ODE method for convergence of stochastic approximation and
reinforcement learning.
\newblock \emph{SIAM Journal on Control and Optimization}, 38(2):447--469,
2000.

\bibitem[Che et~al.(2024)]{che2024target}
Fengdi Che, Chenjun Xiao, Jincheng Mei, Bo Dai, Ramki Gummadi,
Oscar~A. Ramirez, Christopher~K. Harris, A.~Rupam Mahmood, and Dale Schuurmans.
\newblock Target networks and over-parameterization stabilize off-policy
bootstrapping with function approximation.
\newblock In \emph{Proceedings of the 41st International Conference on Machine
Learning}, volume 235 of \emph{Proceedings of Machine Learning Research}, pages
6372--6396. PMLR, 2024.

\bibitem[Chen et~al.(2023)]{chen2023target}
Zaiwei Chen, John-Paul Clarke, and Siva~Theja Maguluri.
\newblock Target network and truncation overcome the deadly triad in Q-learning.
\newblock \emph{SIAM Journal on Mathematics of Data Science}, 5(4):1078--1101,
2023.
\newblock doi:10.1137/22M1499261.

\bibitem[Cicone(2015)]{cicone2015note}
Antonio Cicone.
\newblock A note on the joint spectral radius.
\newblock arXiv preprint arXiv:1502.01506, 2015.

\bibitem[Fellows et~al.(2023)]{fellows2023why}
Mattie Fellows, Matthew~J.~A. Smith, and Shimon Whiteson.
\newblock Why target networks stabilise temporal difference methods.
\newblock In \emph{Proceedings of the 40th International Conference on Machine
Learning}, volume 202 of \emph{Proceedings of Machine Learning Research}, pages
9886--9909. PMLR, 2023.

\bibitem[Hu et~al.(2011)]{hushenzhang2010generating}
Jianghai Hu, Jinglai Shen, and Wei Zhang.
\newblock Generating functions of switched linear systems: analysis,
computation, and stability applications.
\newblock \emph{IEEE Transactions on Automatic Control}, 56(5):1059--1074,
2011.

\bibitem[Jaakkola et~al.(1993)]{jaakkola1994convergence}
Tommi Jaakkola, Michael Jordan, and Satinder Singh.
\newblock Convergence of stochastic iterative dynamic programming algorithms.
\newblock \emph{Advances in Neural Information Processing Systems}, 6, 1993.

\bibitem[Heil and Strang(1995)]{heilstrang1995continuity}
Christopher Heil and Gilbert Strang.
\newblock Continuity of the joint spectral radius: application to wavelets.
\newblock In A.~Bojanczyk and G.~Cybenko, editors, \emph{Linear Algebra for
Signal Processing}, volume 69 of \emph{The IMA Volumes in Mathematics and its
Applications}, pages 51--61. Springer, 1995.

\bibitem[Jungers(2009)]{jungers2009joint}
Rapha{\"e}l Jungers.
\newblock \emph{The Joint Spectral Radius: Theory and Applications}.
\newblock Springer, volume 385, 2009.

\bibitem[Kreyszig(1978)]{kreyszig1978introductory}
Erwin Kreyszig.
\newblock \emph{Introductory Functional Analysis with Applications}.
\newblock John Wiley \& Sons, New York, 1978.

\bibitem[Lee(2026)]{lee2026lyapunovcertified}
Donghwan Lee.
\newblock Lyapunov-certified direct switching theory for Q-learning.
\newblock arXiv preprint arXiv:2604.19569, 2026.

\bibitem[Lee and He(2019)]{leehe2019target}
Donghwan Lee and Niao He.
\newblock Target-based temporal-difference learning.
\newblock In \emph{Proceedings of the 36th International Conference on Machine
Learning}, volume 97 of \emph{Proceedings of Machine Learning Research}, pages
3713--3722. PMLR, 2019.

\bibitem[Lee and He(2020)]{leehe2020periodic}
Donghwan Lee and Niao He.
\newblock Periodic Q-learning.
\newblock In \emph{Proceedings of the 2nd Conference on Learning for Dynamics
and Control}, volume 120 of \emph{Proceedings of Machine Learning Research},
pages 582--598. PMLR, 2020.

\bibitem[Lee et~al.(2023)]{leehuhe2023discrete}
Donghwan Lee, Jianghai Hu, and Niao He.
\newblock A discrete-time switching system analysis of Q-learning.
\newblock \emph{SIAM Journal on Control and Optimization}, 61(3):1861--1880,
2023.

\bibitem[Lee and Lim(2026)]{LeeLim2026}
Donghwan Lee and Han-Dong Lim.
\newblock A switching system theory of Q-learning with linear function
approximation.
\newblock arXiv preprint arXiv:2605.11021, 2026.
\newblock \url{https://arxiv.org/pdf/2605.11021}

\bibitem[Liberzon(2003)]{liberzon2003switching}
Daniel Liberzon.
\newblock \emph{Switching in Systems and Control}.
\newblock Springer Science \& Business Media, 2003.

\bibitem[Lillicrap et~al.(2016)]{LillicrapEtAl2015DDPG}
Timothy~P. Lillicrap, Jonathan~J. Hunt, Alexander Pritzel, Nicolas Heess,
Tom Erez, Yuval Tassa, David Silver, and Daan Wierstra.
\newblock Continuous control with deep reinforcement learning.
\newblock In \emph{Proceedings of the 4th International Conference on Learning
Representations}, 2016.
\newblock arXiv:1509.02971.

\bibitem[Lim and Lee(2025)]{limlee2025understanding}
Han-Dong Lim and Donghwan Lee.
\newblock Understanding the theoretical properties of projected Bellman
equation, linear Q-learning, and approximate value iteration.
\newblock arXiv preprint arXiv:2504.10865, 2025.

\bibitem[Lin and Antsaklis(2009)]{lin2009stability}
Hai Lin and Panos~J. Antsaklis.
\newblock Stability and stabilizability of switched linear systems: A survey of
recent results.
\newblock \emph{IEEE Transactions on Automatic Control}, 54(2):308--322, 2009.

\bibitem[Meyn(2024)]{meyn2024projected}
Sean~P. Meyn.
\newblock The projected Bellman equation in reinforcement learning.
\newblock \emph{IEEE Transactions on Automatic Control}, 69(12):8323--8337,
2024.
\newblock doi:10.1109/TAC.2024.3409647.

\bibitem[Mnih et~al.(2015)]{MnihEtAl2015}
Volodymyr Mnih, Koray Kavukcuoglu, David Silver, Andrei~A. Rusu, Joel Veness,
Marc~G. Bellemare, Alex Graves, Martin Riedmiller, Andreas~K. Fidjeland,
Georg Ostrovski, Stig Petersen, Charles Beattie, Amir Sadik, Ioannis Antonoglou,
Helen King, Dharshan Kumaran, Daan Wierstra, Shane Legg, and Demis Hassabis.
\newblock Human-level control through deep reinforcement learning.
\newblock \emph{Nature}, 518:529--533, 2015.

\bibitem[Puterman(1994)]{puterman2014markov}
Martin~L. Puterman.
\newblock \emph{Markov Decision Processes: Discrete Stochastic Dynamic
Programming}.
\newblock John Wiley \& Sons, New York, 1994.

\bibitem[Rota and Strang(1960)]{rota1960note}
Gian-Carlo Rota and Gilbert Strang.
\newblock A note on the joint spectral radius.
\newblock \emph{Indagationes Mathematicae}, 22(4):379--381, 1960.

\bibitem[Shorten et~al.(2007)]{shorten2007stability}
Robert Shorten, Fabian Wirth, Oliver Mason, Kai Wulff, and Christopher King.
\newblock Stability criteria for switched and hybrid systems.
\newblock \emph{SIAM Review}, 49(4):545--592, 2007.

\bibitem[Sutton and Barto(1998)]{sutton1998reinforcement}
Richard~S. Sutton and Andrew~G. Barto.
\newblock \emph{Reinforcement Learning: An Introduction}.
\newblock MIT Press, 1998.

\bibitem[Tsitsiklis(1994)]{tsitsiklis1994asynchronous}
John~N. Tsitsiklis.
\newblock Asynchronous stochastic approximation and Q-learning.
\newblock \emph{Machine Learning}, 16(3):185--202, 1994.

\bibitem[Watkins and Dayan(1992)]{watkins1992q}
Christopher~J.~C.~H. Watkins and Peter Dayan.
\newblock Q-learning.
\newblock \emph{Machine Learning}, 8(3):279--292, 1992.

\bibitem[Zhang et~al.(2021)]{zhang2021breaking}
Shangtong Zhang, Hengshuai Yao, and Shimon Whiteson.
\newblock Breaking the deadly triad with a target network.
\newblock In \emph{Proceedings of the 38th International Conference on Machine
Learning}, volume 139 of \emph{Proceedings of Machine Learning Research}, pages
12621--12631. PMLR, 2021.

\end{thebibliography}
\end{document}